\documentclass[mathpazo]{arxiv_v2_assets/cicp}
\usepackage{mathtools,bm,etoolbox,multirow,array,placeins,afterpage}
\newcommand{\citep}[1]{\cite{#1}}
\usepackage{algorithm,algpseudocode}
\usepackage{caption,overpic,xcolor}
\usepackage[colorlinks=false,pdfborder={0 0 1},linkbordercolor={0 0 1},citebordercolor={0 0 1},urlbordercolor={0 0 1},hyperfootnotes=false]{hyperref}
\usepackage[nameinlink]{cleveref}
\crefname{figure}{Fig.}{Figs.}\Crefname{figure}{Fig.}{Figs.}
\crefname{section}{Section}{Sections}\Crefname{section}{Section}{Sections}
\crefname{subsection}{Subsection}{Subsections}\Crefname{subsection}{Subsection}{Subsections}
\crefname{equation}{Eq.}{Eqs.}\Crefname{equation}{Eq.}{Eqs.}
\crefname{proposition}{Proposition}{Propositions}\Crefname{proposition}{Proposition}{Propositions}
\crefname{remark}{Remark}{Remarks}\Crefname{remark}{Remark}{Remarks}
\graphicspath{{arxiv_v2_assets/}}
\hypersetup{
  pdftitle={ProFlow: Zero-Shot Physics-Consistent Sampling via Proximal Flow Guidance},
  pdfauthor={Zichao Yu, Ming Li, Wenyi Zhang, Difan Zou, and Weiguo Gao}
}
\allowdisplaybreaks[1]
\floatname{algorithm}{Algorithm}
\algrenewcommand\algorithmicrequire{\textbf{Input:}}
\algrenewcommand\algorithmicensure{\textbf{Output:}}
\algrenewcommand\alglinenumber[1]{\normalfont #1:}
\newcommand{\eqr}[1]{\textup{Eq.~\eqref{#1}}}
\makeatletter
\renewcommand\paragraph{\@startsection{paragraph}{4}{\z@}{6.8pt}{-1em}{\normalfont\normalsize\bfseries}}
\def\thm@space@setup{\thm@preskip=6.8pt\thm@postskip=6.8pt}
\def\thmhead#1#2#3{\thmname{#1}\thmnumber{\@ifnotempty{#1}{ }\@upn{#2\@ifnotempty{#3}{.}}}\thmnote{ {\the\thm@notefont(#3)}}}

\let\ProductionCompactSection\section
\let\ProductionAppendixASection\section

\let\ProductionModerateSubsection\subsection
\newenvironment{ProductionMetricsList}{\let\ProductionListI\@listi\def\@listi{\ProductionListI\topsep1pt\partopsep0pt}\begin{itemize}}{\end{itemize}}

\let\ProductionCompactSubsection\subsection
\newcommand{\ProductionTitleRule}{\noindent\kern.001pt\hrulefill\kern.001pt}
\patchcmd{\@maketitle}{\noindent\hrulefill}{\ProductionTitleRule}{}{\PackageError{production}{First title rule patch failed}{}}
\patchcmd{\@maketitle}{\noindent\hrulefill}{\ProductionTitleRule}{}{\PackageError{production}{Second title rule patch failed}{}}
\newcommand{\ProductionAlgorithmRule}[1][.8pt]{\nointerlineskip\hbox{\kern66sp\vrule height#1 width\dimexpr\linewidth-132sp\relax\kern66sp}\nointerlineskip}
\newcommand\fs@cicpalgorithm{\def\@fs@cfont{\normalfont\normalsize}\let\@fs@capt\floatc@ruled\def\@fs@pre{\ProductionAlgorithmRule\kern3.4pt}\def\@fs@post{\kern3.4pt\ProductionAlgorithmRule}\def\@fs@mid{\kern3.4pt\ProductionAlgorithmRule[.4pt]\kern3.4pt}\let\@fs@iftopcapt\iftrue}
\floatstyle{cicpalgorithm}\restylefloat{algorithm}
\AtBeginDocument{\captionsetup[algorithm]{font={normalsize,rm},labelfont={normalsize,rm,bf},textfont={normalsize,rm},labelsep=space,singlelinecheck=false}}
\apptocmd{\thebibliography}{\clubpenalty=10000\@clubpenalty=10000\widowpenalty=10000\rightskip=.1pt}{}{}
\newcommand{\ProductionAppendixHead}{\renewcommand\@seccntformat[1]{\ifstrequal{##1}{section}{Appendix~\csname the##1\endcsname\quad}{\csname the##1\endcsname\quad}}}
\patchcmd{\@maketitle}{\vfil\null}{\null}{}{}
\apptocmd{\@maketitle}{\unskip\par}{}{}
\let\ps@firstpage\ps@plain
\def\ps@headings{%
  \let\@oddfoot\@empty\let\@evenfoot\@empty
  \def\@oddhead{\footnotesize ProFlow: Zero-Shot Physics-Consistent Sampling\hfil\thepage}%
  \def\@evenhead{\footnotesize\thepage\hfil Z. Yu et al.}}
\AtBeginDocument{\setlength{\textfloatsep}{22.7964pt}\setlength{\floatsep}{12.649pt}\setlength{\intextsep}{12.694pt}}
\makeatother
\setboolean{ONE@page}{true}
\begin{document}
\title{\raggedright\hyphenpenalty=10000\exhyphenpenalty=10000 ProFlow: Zero-Shot Physics-Consistent Sampling\\via Proximal Flow Guidance}
\author[Z. Yu et al.]{\mbox{Zichao Yu}\affil{1}\textsuperscript{$\dagger$}, \mbox{Ming Li}\affil{2}\textsuperscript{$\dagger$}, \mbox{Wenyi Zhang}\affil{3}, \mbox{Difan Zou}\affil{4} and \mbox{Weiguo Gao}\affil{5}\comma\corrauth}
\address{\raggedright\hyphenpenalty=10000\exhyphenpenalty=10000
\affilnum{1}\hspace{0.5em}School of Computing and Data Science, The University of Hong Kong, Hong Kong, China.\\
\affilnum{2}\hspace{0.5em}School of Mathematical Sciences, Fudan University, Shanghai 200433, China.\\
\affilnum{3}\hspace{0.5em}Department of Electronic Engineering and Information Science, University of Science\\ and Technology of China, Hefei, Anhui 230027, China.\\
\affilnum{4}\hspace{0.5em}School of Computing and Data Science \& Shenzhen Loop Area Institute, The University of Hong Kong, Hong Kong, China.\\
\affilnum{5}\hspace{0.5em}School of Mathematical Sciences, Fudan University, Shanghai 200433, China \& Shanghai\\ Key Laboratory of Contemporary Applied Mathematics, Shanghai 200433, China.}
\emails{{\tt zichaoyu1@connect.hku.hk} (Z. Yu), {\tt mingli23@m.fudan.edu.cn} (M. Li), {\tt wenyizha@ustc.edu.cn} (W. Zhang), {\tt dzou@cs.hku.hk} (D. Zou), {\tt wggao@fudan.edu.cn} (W. Gao).\newline $\dagger$ Zichao Yu and Ming Li contributed equally to this work.}
\begin{abstract}
\setlength{\rightskip}{.1pt}
Inferring physical fields from sparse observations while closely respecting the governing partial differential equations is a fundamental challenge in computational physics. Deep generative models offer powerful data-driven priors for such inverse problems, yet existing methods struggle to enforce hard physical constraints without costly retraining or disrupting the learned generative prior. This calls for a sampling mechanism that reconciles strict physical consistency and observational fidelity with the statistical structure of the pre-trained prior. To this end, we present ProFlow, a proximal guidance framework for zero-shot physics-consistent sampling, defined as inferring solutions from sparse observations using a fixed generative prior without task-specific retraining. The algorithm alternates between two steps: (i) a terminal optimization step, which drives the flow prediction toward the intersection of the physically and observationally consistent sets via proximal minimization; and (ii) an interpolation step, which maps the refined state back to the generative trajectory to maintain consistency with the learned flow probability path. This procedure admits a Bayesian interpretation as a sequence of local maximum a posteriori updates. Benchmarks on Poisson, Helmholtz, Darcy, and viscous Burgers' equations demonstrate that ProFlow achieves superior physical and observational consistency and more accurate distributional statistics than state-of-the-art diffusion- and flow-based baselines.
Our code is publicly available at \href{https://github.com/yzc-666/ProFlow}{this ProFlow repository}.
\end{abstract}
\let\ProductionFootnoteRule\footnoterule
\renewcommand{\footnoterule}{\noindent\rule{\dimexpr\textwidth-.28pt\relax}{0.25pt}\vspace*{6pt}}

\maketitle
\setlength{\topskip}{10.75pt}
% Positive bearing reserves keep Palatino ink inside the unchanged body frame.
\setlength{\leftskip}{.18pt}
\setlength{\rightskip}{.1pt}

\afterpage{\global\let\footnoterule\ProductionFootnoteRule}
\section{Introduction}
\label{sec:introduction}
Many problems in computational physics involve inferring unknown fields that are constrained by partial differential equations (PDEs) and only partially observed in space or time~\citep{kaipio2005statistical,stuart2010inverse,tarantola2005inverse}. Examples include recovering material properties from sparse measurements~\citep{calderon2006inverse,sylvester1987global}, assimilating sensor data into a physical model~\citep{evensen2009data,stuart2015data}, or generating ensembles of plausible states for uncertainty quantification~\citep{smith2024uncertainty}. Bridging this gap between theoretical governing equations and limited empirical observations is fundamental to predictive modeling and scientific discovery.

Formally, we consider a family of PDEs posed on a bounded spatial domain $\Omega\subset\mathbb{R}^d$. Let $\bm u\colon\Omega\to\mathbb{R}^n$ denote the state field in a suitable function space $\mathcal{U}$, and let $\mathcal{L}$ be the differential operator encoding the governing equations. We frame the sampling task in a probabilistic setting where $\mathcal{C}$ and $\mathcal{O}$ denote specific events constraining the unknown field. First, let $\mathcal{C}$ denote the \emph{physical consistency event} that the field satisfies the physical laws:
\[
\mathcal{C}\coloneqq\{\mathcal{L}(\bm u)=\bm0\}.
\]
Second, to account for partial and noisy information, we introduce an \emph{observation event} $\mathcal{O}$. Let $\mathcal{H}$ be a measurement operator and $\bm y$ be the observed data. We postulate that the observation is generated via a Gaussian noise model:
\begin{equation}\label{eq:observation_model}
\bm y=\mathcal{H}[\bm u]+\bm\xi,\quad \bm\xi\sim\mathcal{N}(\bm0,\sigma_{\mathrm{obs}}^2\bm I).
\end{equation}
We define $\mathcal{O}$ as the event that the measurement $\bm y$ is observed according to this model. The goal is to sample from the posterior distribution conditioned on the joint occurrence of these events:
\begin{equation}\label{eq:posterior}
p(\bm u\mid\mathcal{C},\mathcal{O})\propto p(\bm u)p(\mathcal{O}\mid\bm u)\bm{1}_{\mathcal{C}}(\bm u).
\end{equation}
Here, $p(\mathcal{O}\mid\bm u)$ is the likelihood induced by the observation model, $\bm{1}_{\mathcal{C}}$ enforces the physical constraint, and $p(\bm u)$ is the generative prior which we will specify later. This framework encompasses a wide spectrum of problems depending on the nature of $\mathcal{H}$. For forward problems, $\mathcal{H}$ restricts the solution space to those satisfying prescribed initial and boundary conditions, whereas for inverse problems, $\mathcal{H}$ encodes sparse observations from which the physical parameters must be inferred.

\begingroup
Classically, these problems are addressed using the Finite Difference Method~\citep{leveque2007finite}, Finite Element Method~\citep{brenner2008mathematical}, Finite Volume Method~\citep{eymard2000finite}, or spectral methods~\citep{gottlieb1977numerical}. Although highly accurate, these traditional approaches require careful design of numerical schemes and are computationally expensive in high dimensions. Consequently, deep learning-based solvers have emerged as a viable alternative, ranging from Physics-Informed Neural Networks (PINNs)~\citep{raissi2019physics} to operator learners like DeepONet~\citep{lu2021learning} and Fourier Neural Operators (FNO)~\citep{li2021fourier}. Despite their utility, these methods often suffer from optimization challenges and limited expressiveness. This has prompted a pivot toward generative paradigms such as diffusion models~\citep{ho2020denoising,song2021score,yu2025tree} and flow matching models~\citep{lipman2023flow,liu2023flow}. They have been used to synthesize realistic images and videos~\citep{dhariwal2021diffusion,esser2024scaling,kong2024hunyuanvideo,wan2025wan}, design molecules and materials~\citep{hoogeboom2022equivariant,zeni2025generative}, and emulate high-dimensional climate and weather fields~\citep{bassetti2024diffesm,landry2025generating}, among many other examples. In the realm of PDEs, generative models can learn complex distributions and perform \emph{zero-shot}\footnote{In this context, zero-shot refers to satisfying novel observation constraints or boundary conditions at sampling time without retraining the pretrained model (i.e., the generative model that approximates the prior distribution on the physically consistent set).} controllable generation tasks in the sampling process~\citep{cheng2025gradientfree,huang2024diffusionpde,utkarsh2025physics}. For example, DiffusionPDE~\citep{huang2024diffusionpde} leverages the gradients of the PDE residual to guide the diffusion sampling process, while D-Flow~\citep{ben2024d} optimizes the initial noise for better alignment with PDE constraints. Conversely, the ECI framework~\citep{cheng2025gradientfree} and physics-constrained flow matching models~\citep{utkarsh2025physics} circumvent the high cost of backpropagating through solvers by using manifold projections or linear corrections to align the generative trajectory with physical constraints. However, these approaches face common limitations. First, they typically enforce physical validity only as soft constraints, often via penalty terms or approximate guidance. As a result, the generated solutions are not guaranteed to strictly satisfy the governing PDE, leading to potential violations of conservation laws or physical consistency. Second, they often lack a clear derivation from a principled Bayesian formulation. Instead of following a strict inference procedure, current approaches largely rely on heuristic trajectory modifications, such as ad hoc projections or gradient penalties, to enforce consistency. Consequently, these adjustments are not guaranteed to target the true posterior distribution $p(\bm u\mid\mathcal{C},\mathcal{O})$. Without this explicit link to a well-defined conditional density, it is difficult to ensure that the generated samples statistically represent the true uncertainty of the solution.

\par\endgroup

To resolve the aforementioned limitations, we propose \textbf{ProFlow}, a proximal guidance scheme for zero-shot physics-consistent sampling with pretrained generative priors. We build upon the framework of Functional Flow Matching (FFM)~\citep{kerrigan2024functional} as a generative prior because it fundamentally treats the data as continuous fields rather than discrete vectors. Specifically, FFM learns a continuous flow on the function space $\mathcal{U}$ that transports a simple reference measure $\mu_0$ (typically a Gaussian Random Field) to the target distribution $\mu_1$ supported on the physically consistent set. This is achieved by training a neural network to approximate the velocity field of a conditional probability path, which is typically defined as a linear interpolation between the reference and target samples. The terminal distribution $\mu_1$ thus acts as a learned prior density $p(\bm u)$ over the admissible functions in $\mathcal{U}$. Conditioning this prior on the physics and observation constraints yields the target posterior in \eqr{eq:posterior}. As illustrated in \cref{fig:conceptual_illustration}, ProFlow modifies the sampling trajectory to sample from the constrained posterior. At each timestep $t_n$ (with $n$ indexing the discretization steps), the algorithm iterates through two complementary steps: (i) a terminal optimization step where the FFM model predicts a candidate terminal state $\hat{\bm{u}}_1=\bm{u}_t+(1-t)\bm{v}_\theta(\bm{u}_t,t)$. This candidate is then refined into a valid solution $\bm{u}_1$ by solving a proximal optimization problem:
\begingroup
\[
\bm{u}_1=\operatorname*{arg\,min}_{\bm{u}\in\mathcal{C}}\|\bm{u}-\hat{\bm{u}}_1\|_2^2+\lambda\|\mathcal{H}[\bm{u}]-\bm{y}\|_2^2.
\]
This step enforces the observation operator $\mathcal{H}$ and the constraints $\mathcal{C}$ while keeping the solution anchored to the generative prior's prediction; and (ii) an interpolation step where the refined state $\bm{u}_1$ is then mapped back to the next intermediate time $t_{n+1}$ via linear interpolation with the freshly drawn noise $\bm{\varepsilon}$:
\[
\bm{u}_{t_{n+1}}=(1-t_{n+1})\bm{\varepsilon}+t_{n+1}\bm{u}_1.
\]\endgroup
This ensures the update step remains consistent with the linear probability path (the straight-line bridge) assumed during the FFM training. By iterating these steps, ProFlow provides a rigorous mechanism to enforce both physical and observational constraints while preserving the statistical structure of the learned prior.

\begin{figure}[!b]
\centering\captionsetup{skip=7.6175pt,belowskip=.0475pt}\includegraphics[pagebox=cropbox,width=\linewidth]{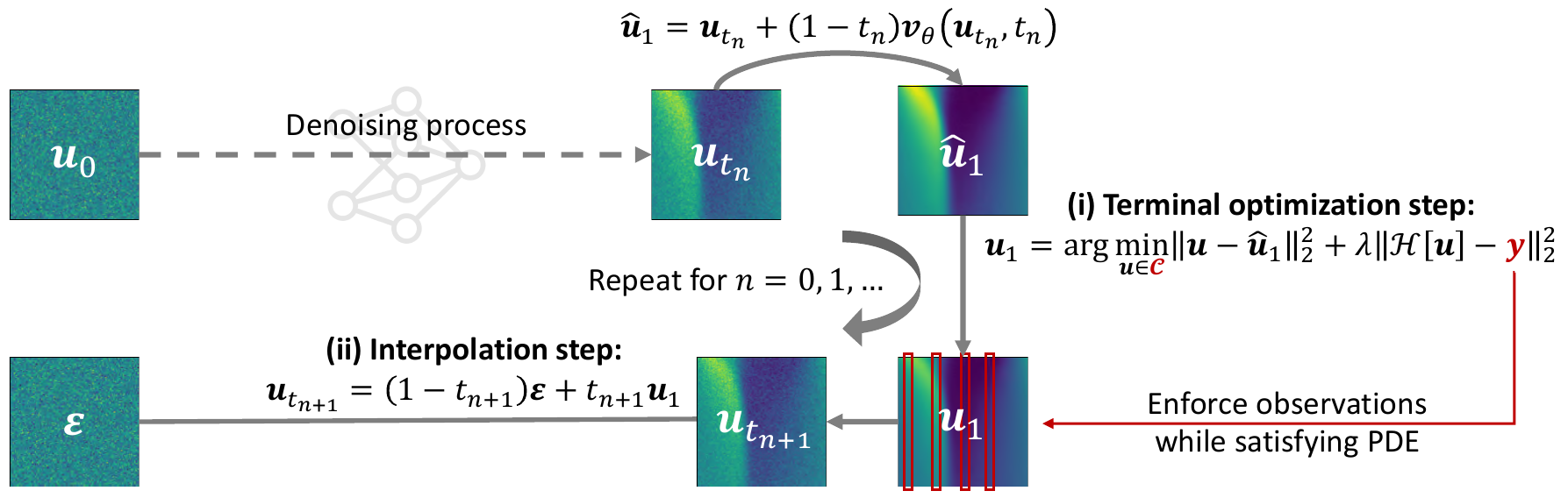}
\caption{A conceptual illustration of ProFlow. The algorithm alternates between (i) a terminal optimization step, which projects the model's predicted solution $\hat{\bm{u}}_1$ onto the physically and observationally consistent set by solving a proximal optimization problem; and (ii) an interpolation step, which constructs the next state $\bm{u}_{t_{n+1}}$ by linearly blending the refined solution $\bm{u}_1$ with the freshly drawn noise $\bm{\varepsilon}$, adhering to the flow matching probability path.}
\label{fig:conceptual_illustration}
\end{figure}

\begingroup\tolerance=1000\looseness=-1
The remainder of the paper is organized as follows. \Cref{sec:related_work} reviews related developments in physics-informed and generative modeling for PDEs. \Cref{sec:ffm_for_pdes} provides necessary background on the functional flow matching prior used in this work. \Cref{sec:method} presents the probabilistic formulation, derives the proximal guidance scheme, and describes the ProFlow algorithm. \Cref{sec:experiments} reports numerical results on elliptic and time-dependent PDE benchmarks and compares ProFlow with representative baselines. Finally, \Cref{sec:conclusion} concludes with a summary and perspectives for future work.

\par\endgroup

\section{Related works and our contributions}\label{sec:related_work}
\subsection{Classical numerical solvers}
Partial differential equations (PDEs) serve as the fundamental mathematical framework for describing physical systems across space and time. However, practical governing equations rarely admit closed-form analytical solutions. Consequently, scientific computing has long relied on numerical discretization schemes such as the Finite Difference Method~\citep{leveque2007finite}, Finite Element Method~\citep{brenner2008mathematical}, Finite Volume Method~\citep{eymard2000finite}, and spectral methods~\citep{gottlieb1977numerical} to approximate solutions. These classical solvers operate by discretizing the continuous domain into meshes or grids, converting differential operators into systems of algebraic equations that can be solved iteratively. While mathematically rigorous and ubiquitous in engineering, these methods face great challenges at fine resolutions or in high-dimensional settings due to prohibitive computational costs. Additionally, they typically require substantial problem-specific engineering, such as complex mesh generation, to accommodate intricate geometries. Furthermore, their high computational overhead makes them ill-suited for many-query tasks like inverse analysis or uncertainty quantification, motivating the recent rapid shift toward deep learning-based approaches.

\subsection{Physics-informed deep learning}
Motivated by the computational bottlenecks of classical methods, recent research has shifted toward deep learning-based solvers. Pioneering this direction, Physics-Informed Neural Networks (PINNs)~\citep{raissi2019physics} use coordinate-based neural networks that embed governing equations directly into the training objective via PDE residuals and boundary penalties. This formulation allows for the learning of continuous solution fields in a mesh-free manner, enabling data-efficient modeling even when observations are sparse. Consequently, PINNs have been successfully adapted to a wide variety of physical domains including fluid dynamics~\citep{cai2021physics,mao2020physics}, Reynolds-averaged Navier--Stokes modeling~\citep{eivazi2022physics}, and heat transfer~\citep{cai2021physicsheat}. Despite their flexibility, PINNs face scalability challenges since each new problem instance or parameter configuration requires solving a high-dimensional non-convex optimization problem from scratch. This iterative training burden becomes increasingly prohibitive for long-horizon dynamics or stiff differential equations, limiting their utility for real-time applications.

\subsection{Neural operator learning}
To address the single-instance limitation of PINNs, operator-learning frameworks were proposed to approximate mappings between infinite-dimensional function spaces. Seminal architectures such as the Fourier Neural Operator (FNO)~\citep{li2021fourier} and DeepONet~\citep{lu2021learning} learn solution operators that generalize across variable inputs, allowing for rapid evaluation without retraining. This paradigm has been significantly expanded to handle physical and geometric complexities. For instance, recent works utilize learned coordinate deformations to extend FNOs to general, irregular geometries~\citep{li2023fourier}. Furthermore, the scope of operator learning has moved beyond standard differential equations to more complex integral formulations. Notably, Bassi et al.~\citep{bassi2024learning} employ recurrent architectures to solve integro-differential equations, while Zhu et al.~\citep{zhu2025predicting} apply nonlinear integral operator learning to predict nonequilibrium Green's function dynamics in quantum many-body systems. Related efforts further broaden the scope of operator and data-driven learning: a Causality-DeepONet enforces causality when learning the responses of linear dynamical systems~\citep{liu2024causality}, and governing equations themselves can be identified directly from data sampled on closed surfaces through sparse optimization~\citep{sun2025learning}. Finally, to improve physical fidelity and handle inverse problems, frameworks like PINO~\citep{li2024physics} and iFNO~\citep{long2025invertible} integrate physical constraints or invertibility directly into the operator architecture, bridging the gap between data-driven approximation and physical consistency.

\subsection{Generative modeling for PDEs}
While neural operator learning significantly accelerates forward simulations, these methods are typically deterministic and lack mechanisms to naturally represent uncertainty or assimilate partial observations. Representative of this deterministic paradigm, noise- and vibration-tolerant residual networks have been developed to reconstruct physical fields from sparse sensor measurements~\citep{li2024noise}. To mitigate this, recent work has explored generative priors, ranging from graph-based models~\citep{iakovlev2021learning,zhao2022learning} to deep diffusion processes, to model distributions of admissible fields. Current approaches generally fall into three categories. The first category is training-time regularization, which embeds physics directly into the generative objective function (e.g., Physics-Informed Diffusion Models~\citep{bastek2025physics}, Physics-Based Flow Matching~\citep{baldan2025flow}). While effective, these models cannot adapt to new boundary conditions without retraining. Alternatively, fine-tuning approaches such as Physics-Informed Reward Fine-tuning~\citep{yuan2025pirf} adapt model weights to new observations. However, this process is computationally expensive and risks catastrophic forgetting, where the model deviates from the original learned distribution. The third category, sampling-time guidance, avoids retraining by using fixed priors to enforce constraints during sampling~\citep{gao2026terminally}. Our proposed method belongs to this category. Existing works in this area differ by their enforcement mechanism. For example, DiffusionPDE~\citep{huang2024diffusionpde} applies soft, gradient-based guidance using PDE residuals. While flexible, it cannot guarantee strict physical feasibility. Conversely, ECI~\citep{cheng2025gradientfree} and PCFM~\citep{utkarsh2025physics} employ hard projections. ECI is restricted to linear constraints, while PCFM utilizes Gauss--Newton updates. More recently, HardFlow~\citep{li2025hardflow} casts hard-constrained sampling for flow-matching models as a trajectory-optimization problem that enforces the constraint at the terminal time. Despite their success, these methods face a fundamental trade-off. Soft guidance often permits violations of physical laws, whereas hard projections can create inconsistencies with the generative probability path, reducing the statistical quality of the samples. Furthermore, many existing approaches prioritize empirical adjustments to the sampling trajectory. Consequently, deriving these methods from a unified probabilistic sampling objective to establish theoretical guarantees remains an open challenge.

\ProductionCompactSubsection{Contributions of this work}
To address these challenges, we introduce \textbf{ProFlow}, which resolves the tension between soft guidance and hard projection by unifying strict physical constraints with generative fidelity. We frame the problem within a principled Bayesian formulation, deriving a conditional law factorization that yields two interpretable updates: a proximal terminal refinement that enforces physics and data constraints, and a flow-consistent interpolation that restores proximity to the pretrained generative manifold. This approach unifies guidance and projection methods under a theoretically grounded, zero-shot sampling framework. We highlight the key distinctions between ProFlow and existing methods in \Cref{tab:comp}. The main contributions of this work are as follows:
\begin{itemize}
\item A Bayesian formulation of zero-shot physics-informed sampling with flow matching priors, expressed as a constrained posterior $p(\bm u\mid\mathcal{C},\mathcal{O})$ and a conditional law factorization that motivates a two-stage guidance scheme.
\item A terminal proximal optimization step that admits an interpretation as a local MAP estimate for an approximate conditional model, balancing proximity to the flow prediction with adherence to physics and data consistency in a single optimization problem.
\item An interpolation step that restores compatibility with the learned FFM bridges after each proximal update, ensuring the pretrained velocity field operates within its valid domain rather than correcting off-manifold states.
\item A comprehensive numerical study on Poisson, Helmholtz, and Darcy flow benchmarks and viscous Burgers trajectories, demonstrating that ProFlow achieves physical and observational consistency, as well as more accurate distributional statistics compared with baseline methods.
\end{itemize}
\begin{table}[tbp]
\centering\renewcommand{\arraystretch}{1.20}
\caption{Comparison of generative methods incorporating physical or observational constraints. Columns indicate: zero-shot (sampling without retraining), observation/physics (ability to enforce respective constraints), prior (preservation of the pretrained generative manifold statistics), and enforcement type (mechanism of constraint handling).}\label{tab:comp}
\setlength{\tabcolsep}{3.29pt}
\begin{cicptabular}{l|c|c|c|c|l}\hline
Method & Zero-shot & Observation & Physics & Prior & Enforcement type\\\hline
FFM~\citep{kerrigan2024functional} & $\times$ & $\times$ & $\times$ & $\checkmark$ & N/A\\
DiffusionPDE~\citep{huang2024diffusionpde} & $\checkmark$ & $\checkmark$ & $\checkmark$ & $\checkmark$ & Soft (Gradient guidance)\\
D-Flow~\citep{ben2024d} & $\checkmark$ & $\checkmark$ & $\checkmark$ & $\checkmark$ & Soft (Initial noise optimization)\\
ECI~\citep{cheng2025gradientfree} & $\checkmark$ & $\checkmark$ & $\times$ & $\times$ & Hard (Linear projection)\\
PCFM~\citep{utkarsh2025physics} & $\checkmark$ & $\checkmark$ & $\checkmark$ & $\times$ & Hard (Gauss--Newton projection)\\
\textbf{ProFlow (Ours)} & $\checkmark$ & $\checkmark$ & $\checkmark$ & $\checkmark$ & \textbf{Hard (Proximal optimization)}\\\hline
\end{cicptabular}
\end{table}
\section{Functional flow matching prior for PDEs}\label{sec:ffm_for_pdes}
In zero-shot physics-informed sampling, we aim to generate fields that satisfy the PDE and agree with sparse observations, without requiring paired supervision. Instead of a deterministic solver, we learn a \emph{generative prior} over physically consistent fields. This approach yields diverse, approximately plausible samples that we can then guide to enforce the PDE or match new observations.

To construct this prior, we adopt Functional Flow Matching (FFM)~\citep{kerrigan2024functional}. FFM learns a deterministic continuous transformation that smoothly morphs random functions into physical solutions. Formally, let $\mathcal{U}$ be a real separable Hilbert space of functions $\bm u\colon\Omega\to\mathbb{R}^n$ equipped with the Borel $\sigma$-algebra $\mathcal{B}(\mathcal{U})$. FFM defines a flow $\psi_t\colon\mathcal{U}\to\mathcal{U}$ where $t$ is a homotopy variable (not physical time). The generation process starts by sampling a function $\bm u_0$ from a simple reference distribution $\mu_0$ (typically a Gaussian Random Field) and evolving it according to the ODE:
\begin{equation}\label{eq:ffm_ode}
\frac{\mathrm d}{\mathrm dt}\psi_t(\bm u_0)=\bm v_\theta(\psi_t(\bm u_0),t),\quad \psi_0(\bm u_0)=\bm u_0,
\end{equation}
where $\bm v_\theta$ is a learnable velocity field neural network (with $\theta$ denoting the network parameters).

A key advantage of FFM is \emph{simulation-free} training. We do not need to solve the ODE to train the model. Instead, we define a target probability path that interpolates linearly between the prior $\mu_0$ and the data $\mu_1$. The target velocity for any sample pair is simply the straight-line vector $\bm u_1-\bm u_0$. We train $\bm v_\theta$ to match this vector field by minimizing
\[
L_{\mathrm{FFM}}(\theta)=\mathbb{E}_{t\sim U(0,1),\bm u_0\sim\mu_0,\bm u_1\sim\mu_1}\bigl[\|\bm v_\theta(\bm u_t,t)-(\bm u_1-\bm u_0)\|_2^2\bigr],
\]
where $\bm u_t=(1-t)\bm u_0+t\bm u_1$. At sampling time, we generate novel fields by sampling $\bm u_0\sim\mu_0$ and numerically integrating \eqr{eq:ffm_ode} to $t=1$. Because this prior is learned from data, generated samples are only approximately physically consistent. At sampling time, we guide the generative trajectory using both the physical constraint $\mathcal{L}(\bm u)=0$ and the available observations, thereby enforcing PDE consistency while conditioning on data.

\section{Method}\label{sec:method}
\subsection{Problem statement}
We build upon the probabilistic formulation introduced in \Cref{sec:introduction}. Let $\mu_1$ denote the distribution induced by the pretrained FFM model at time $t=1$, which serves as our generative prior $p(\bm u)$. We seek to sample from the posterior distribution conditioned on the physical consistency event $\mathcal{C}$ and the observation event $\mathcal{O}$. Recalling the decomposition in \eqr{eq:posterior}, the target density is
\[
p(\bm u\mid\mathcal{C},\mathcal{O})\propto p(\bm u)p(\mathcal{O}\mid\bm u)\bm{1}_{\mathcal{C}}(\bm u),
\]
where $p(\mathcal{O}\mid\bm u)$ is the likelihood induced by the Gaussian noise postulate in \eqr{eq:observation_model} and $\bm{1}_{\mathcal{C}}$ is the indicator function for the PDE constraint. Exact sampling from this posterior is intractable due to the implicit nature of the prior $p(\bm u)$. Consequently, as detailed in \Cref{sec:derivation_proflow}, we approximate the sampling trajectory via a sequence of local maximum a posteriori (MAP) updates. These updates explicitly balance the flow-based prior, the data likelihood, and the hard physical constraints at each timestep.

\begin{remark}[Extension to time-dependent problems]\label{rem:time_dependent_extension}
While the notation in this section describes stationary fields for clarity, the framework extends naturally to time-dependent settings. For evolution equations, such as Burgers' equation, the variable $\bm u$ represents the entire solution trajectory over the space-time domain $\Omega\times[0,T]$. Consequently, the physical consistency event $\mathcal{C}$ encapsulates the time-evolution residual, while the observation event $\mathcal{O}$ generalizes to measurements distributed across space-time coordinates (e.g., sparse sensors or temporal snapshots). This allows the posterior factorization and sampling procedure to remain structurally identical to the stationary case.
\end{remark}

\ProductionCompactSubsection{Derivation of the proximal guidance scheme}\label{sec:derivation_proflow}
Our goal is to construct a sampling algorithm that approximates the constrained posterior $p(\bm u\mid\mathcal{C},\mathcal{O})$ via a sequence of tractable local transitions. As described in \Cref{sec:ffm_for_pdes}, the pretrained flow matching model defines a homotopy time variable $t\in[0,1]$ that interpolates between a reference distribution $\mu_0$ at $t=0$ and the prior $\mu_1$ at $t=1$. We write $\bm u_t$ for the state at homotopy time $t$.

\begingroup\addtolength{\abovedisplayskip}{6.5pt}\addtolength{\belowdisplayskip}{6.5pt}\addtolength{\belowdisplayshortskip}{6.5pt}
Conceptually, conditioning the terminal distribution on the constraints defines a posterior path $\{p(\bm u_t\mid\mathcal{C},\mathcal{O})\}_{t\in[0,1]}$ obtained by transporting the constrained terminal law backward along the learned flow. Suppose that at some time $t$ we have access to a draw $\bm u_t\sim p(\bm u_t\mid\mathcal{C},\mathcal{O})$. For any later time $s\in(t,1]$ the corresponding marginal can be written by conditioning on intermediate states, as shown in \Cref{prop:conditional_law_factorization}.

\begin{proposition}[Conditional law factorization]\label{prop:conditional_law_factorization}
\itshape\leftskip=.1pt\rightskip=.1pt
Let $0\leq t<s\leq1$ and let $\mathcal{C}$ and $\mathcal{O}$ be the physical consistency event and observation event, respectively. Assume that under the conditional measure $\mathbb{P}(\cdot\mid\mathcal{C},\mathcal{O})$ the joint law of $(\bm u_t,\bm u_s,\bm u_1)$ admits conditional densities. Then
\begingroup\predisplaypenalty=0
\[
p(\bm u_s\mid\mathcal{C},\mathcal{O})=\iint p(\bm u_s\mid\bm u_1,\bm u_t,\mathcal{C},\mathcal{O})p(\bm u_1\mid\bm u_t,\mathcal{C},\mathcal{O})
 p(\bm u_t\mid\mathcal{C},\mathcal{O})\,\mathrm d\bm u_1\,\mathrm d\bm u_t.
\]\endgroup
\end{proposition}
\begin{proof}
Let $\mathcal E=\mathcal{C}\cap\mathcal{O}$ denote the joint event of satisfying physical and observational constraints. We start with the marginal density $p(\bm u_s\mid\mathcal E)$, which can be obtained by integrating the joint density $p(\bm u_s,\bm u_1,\bm u_t\mid\mathcal E)$ over the variables $\bm u_1$ and $\bm u_t$:
\[
p(\bm u_s\mid\mathcal E)=\iint p(\bm u_s,\bm u_1,\bm u_t\mid\mathcal E)\,\mathrm{d}\bm u_1\,\mathrm{d}\bm u_t.
\]
By the chain rule of probability, the joint density admits the factorization
\[
p(\bm u_s,\bm u_1,\bm u_t\mid\mathcal E)=p(\bm u_s\mid\bm u_1,\bm u_t,\mathcal E)p(\bm u_1\mid\bm u_t,\mathcal E)p(\bm u_t\mid\mathcal E).
\]
Substituting this factorization back into the integral yields
\[
p(\bm u_s\mid\mathcal E)=\iint p(\bm u_s\mid\bm u_1,\bm u_t,\mathcal E)p(\bm u_1\mid\bm u_t,\mathcal E)p(\bm u_t\mid\mathcal E)\,\mathrm{d}\bm u_1\,\mathrm{d}\bm u_t.
\]
Replacing $\mathcal E$ with $\mathcal{C},\mathcal{O}$ completes the proof.
\end{proof}

\looseness=2
This factorization shows that posterior samples at later homotopy times can be obtained by hierarchical conditional sampling. Given a current state $\bm u_t\sim p(\bm u_t\mid\mathcal{C},\mathcal{O})$, one first draws a terminal realization $\bm u_1$ from $p(\bm u_1\mid\bm u_t,\mathcal{C},\mathcal{O})$ and then an intermediate state $\bm u_s$ from $p(\bm u_s\mid\bm u_t,\bm u_1,\mathcal{C},\mathcal{O})$. In practice these conditional distributions are intractable and we introduce local approximations that lead to a two-step algorithm. The first step is a local MAP update at the terminal time that enforces the constraints, and the second step is an interpolation step that restores compatibility with the learned flow.

\par\endgroup
\paragraph{Terminal update as a local MAP step.}
We begin with the conditional distribution of the terminal state given the current state and the constraints,
\[
p(\bm u_1\mid\bm u_t,\mathcal{C},\mathcal{O})\propto p(\bm u_1\mid\bm u_t)p(\mathcal{O}\mid\bm u_1)\bm{1}_{\mathcal{C}}(\bm u_1),
\]
where we have used that the physics constraint $\mathcal{C}$ depends only on $\bm u_1$. The learned flow induces a deterministic mapping $\bm u_t\mapsto\bm u_1$, so the exact conditional law $p(\bm u_1\mid\bm u_t)$ is a singular Dirac measure concentrated on the flow trajectory and is not available as a tractable density. Following common practice in diffusion-based inverse solvers~\citep{bastek2025physics,huang2024diffusionpde}, we introduce local stochasticity and approximate this transition by a Gaussian centered at a one-step flow prediction,\looseness=1
\begingroup\predisplaypenalty=0
\begin{equation}
p(\bm u_1\mid\bm u_t)\approx\mathcal{N}\bigl(\hat{\bm u}_1(\bm u_t),\sigma_t^2 I\bigr)\propto\exp\biggl(-\frac{1}{2\sigma_t^2}\|\bm u_1-\hat{\bm u}_1(\bm u_t)\|_2^2\biggr),\label{eq:gaussian_transition}
\end{equation}\endgroup
where
\[
\hat{\bm u}_1(\bm u_t)=\bm u_t+(1-t)\bm v_\theta(\bm u_t,t)
\]
is a forward Euler prediction of the flow and $\sigma_t>0$ is a scalar parameter representing the uncertainty of the one-step flow prediction.

The likelihood term $p(\mathcal{O}\mid\bm u_1)$ follows directly from the Gaussian noise postulate introduced in \eqr{eq:observation_model}. Dropping constant factors, it is given by
\[
p(\mathcal{O}\mid\bm u_1)\propto\exp\biggl(-\frac{1}{2\sigma_\mathrm{obs}^2}\|\mathcal{H}[\bm u_1]-\bm y\|_2^2\biggr).
\]

Combining the Gaussian transition approximation and this observation likelihood yields the following approximation of the negative log conditional posterior:
\begin{align*}
&-\log p(\bm u_1\mid\bm u_t,\mathcal{O},\mathcal{C})\\
\approx{}&\frac{1}{2\sigma_t^2}\|\bm u_1-\hat{\bm u}_1(\bm u_t)\|_2^2
+\frac{1}{2\sigma_\mathrm{obs}^2}\|\mathcal{H}[\bm u_1]-\bm y\|_2^2+\mathrm{const},
\end{align*}
subject to the constraint $\bm u_1\in\mathcal{C}$. The terminal update can therefore be interpreted as a local MAP estimator for this approximate conditional model. In algorithmic form, we absorb the ratio of variances into a single weight $\lambda\propto\sigma_t^2/\sigma_\mathrm{obs}^2$ and solve the proximal problem
\begin{equation}\label{eq:terminal_opt}
\bm u_1^\star\in\arg\min_{\bm u\in\mathcal{C}}\Bigl\{\|\bm u-\hat{\bm u}_1(\bm u_t)\|_2^2+\lambda\|\mathcal{H}[\bm u]-\bm y\|_2^2\Bigr\},
\end{equation}
which balances trust in the flow prediction against trust in the data. The implementation details and the choice of $\lambda$ are deferred to \Cref{subsec:alg_framework}.

\paragraph{Interpolation step and compatibility with the flow.}
The second step targets the conditional distribution
\[
p(\bm u_s\mid\bm u_t,\bm u_1,\mathcal{C},\mathcal{O}).
\]
The homotopy variable $t$ is defined by the generative flow of the pretrained model, and this flow is learned independently of the physical and observation constraints. Once the terminal state $\bm u_1$ has been updated to satisfy the constraints, the intermediate states between $t$ and $1$ should therefore follow the training bridges of the original flow.

Under the flow matching training objective, these bridges are linear optimal transport paths between a reference sample $\bm\varepsilon\sim\mu_0$ and a terminal field $\bm u_1$,
\begingroup\setlength{\belowdisplayskip}{6pt}\setlength{\belowdisplayshortskip}{6pt}
\[
\bm u_\tau=(1-\tau)\bm\varepsilon+\tau\bm u_1,\quad \tau\in[0,1].
\]
If $\bm u_t$ and $\bm u_1$ lie on the same bridge associated with some $\bm\varepsilon$, the intermediate state at time $s$ is given exactly by
\predisplaypenalty=0
\[
\bm u_s=(1-s)\bm\varepsilon+s\bm u_1,\quad\text{for }s>t.
\]\endgroup
After the terminal optimization, however, the updated $\bm u_1^\star$ no longer lies on the same bridge manifold as $\bm u_t$. To restore compatibility with the learned flow we instead draw a fresh reference sample $\bm\varepsilon\sim\mu_0$ and construct a new bridge between $\bm\varepsilon$ and $\bm u_1^\star$. For a discrete time schedule $\{t_n\}_{n=0}^N$ this leads to the interpolation update
\[
\bm u_{t_{n+1}}=(1-t_{n+1})\bm\varepsilon+t_{n+1}\bm u_1^\star.
\]
The fresh reference sample $\bm\varepsilon\sim\mu_0$ plays a specific role here. Since FFM is trained exclusively on linear bridges $(1-\tau)\bm\varepsilon+\tau\bm u_1$ with $\bm\varepsilon\sim\mu_0$, redrawing $\bm\varepsilon$ ensures that the next iterate $\bm u_{t_{n+1}}$ has exactly the marginal law of a training bridge with the refined endpoint $\bm u_1^\star$, so the velocity network $\bm v_\theta(\cdot,t_{n+1})$ is evaluated on an in-distribution input. Omitting the noise, or reusing the previous $\bm\varepsilon$, which is no longer coupled to $\bm u_1^\star$, would place the iterate off the training manifold and bias the subsequent one-step prediction. We remark that similar interpolation strategies have been used in diffusion-based inverse solvers to maintain proximity to the model manifold~\citep{song2024solving}.

\begin{remark}[Status of the Bayesian interpretation]\label{rem:bayesian_status}
\looseness=1
The preceding derivation provides a principled \emph{motivation} for the two updates rather than a proof that ProFlow samples from the exact posterior $p(\bm u\mid\mathcal{C},\mathcal{O})$. Three approximations intervene. First, the singular flow transition $p(\bm u_1\mid\bm u_t)$ is relaxed to the Gaussian surrogate in \eqr{eq:gaussian_transition}, whose variance schedule $\sigma_t$ is prescribed heuristically. Second, and most consequentially, the hierarchical \emph{sampling} of $\bm u_1\sim p(\bm u_1\mid\bm u_t,\mathcal{C},\mathcal{O})$ implied by \Cref{prop:conditional_law_factorization} is replaced by its \emph{mode} through the local MAP problem in \eqr{eq:terminal_opt}. The terminal update therefore returns a point estimate rather than a draw, so posterior spread is not propagated exactly through the recursion. Third, the interpolation step reinjects a fresh reference sample $\bm\varepsilon\sim\mu_0$ to restore compatibility with the training bridges, which is a modeling choice rather than a realization of the exact conditional $p(\bm u_s\mid\bm u_t,\bm u_1,\mathcal{C},\mathcal{O})$. Accordingly, we present the Bayesian formulation as an interpretable framework that makes each design choice explicit and reduces to standard MAP inference in the single-step limit. It is not a claim of exact posterior sampling. We leave the characterization of the stationary law of the resulting iteration and the bounding of its discrepancy from $p(\bm u\mid\mathcal{C},\mathcal{O})$ as an open problem (\Cref{sec:conclusion}), and the distributional statistics reported in \Cref{sec:experiments} should be understood as an empirical assessment of sample quality rather than a guarantee of posterior fidelity.
\end{remark}
\begin{remark}[Continuous-time interpretation]\label{rem:continuous_time}
\tolerance=800
The terminal update is one backward-Euler step, although the full iteration does not reduce to a single continuous-time flow. Collecting the data and physics terms into $g(\bm u)=\lambda\|\mathcal{H}[\bm u]-\bm y\|_2^2+\iota_{\mathcal{C}}(\bm u)$, where $\iota_{\mathcal{C}}$ is the indicator of $\mathcal{C}$, the first-order optimality condition of \eqr{eq:terminal_opt} reads $\bm0\in2(\bm u_1^\star-\hat{\bm u}_1)+\partial g(\bm u_1^\star)$, which is exactly
\[
\bm u_1^\star=(\mathrm{Id}+\tfrac12\partial g)^{-1}(\hat{\bm u}_1)=\operatorname{prox}_{g/2}(\hat{\bm u}_1).
\]
This resolvent is one backward-Euler step, with unit step size, of the subgradient flow $\dot{\bm u}\in-\tfrac12\partial g(\bm u)$ started at $\hat{\bm u}_1$, and the identity is exact. The coupled ProFlow iteration, by contrast, does not admit a naive continuous-time limit under refinement of the schedule, for two reasons. First, the interpolation step reinjects a fresh reference sample $\bm\varepsilon\sim\mu_0$ of order $\mathcal{O}(1)$ at every step, rather than a vanishing $\mathcal{O}(\sqrt{\Delta t})$ increment, so the one-step map does not approach the identity as $\Delta t\to0$ and no trajectory-level ODE or SDE limit follows directly. Second, a single iteration composes three operations acting at different times, an explicit extrapolation to $t=1$, the implicit proximal correction at $t=1$, and a re-bridging back to the intermediate time, so it applies these separate sub-steps in sequence rather than evolving under a single continuous-time equation. What is rigorous is therefore the exact resolvent identity together with this composition of separate sub-steps. A genuine limiting equation would require either a vanishing-noise variant, under which the composition of transport and proximal correction formally approaches a projected gradient flow $\dot{\bm u}\in\bm v_\theta(\bm u,t)-\tfrac12\partial g(\bm u)$ that couples the learned drift with the normal cone of $\mathcal{C}\cap\mathcal{O}$, or an analysis at the level of laws in the stochastic-interpolant framework that yields a Fokker--Planck or continuity equation for the marginals. We do not carry out either derivation here and record it as an open problem.\looseness=1
\end{remark}

\ProductionCompactSubsection{Algorithmic framework of ProFlow}\label{subsec:alg_framework}
We now introduce ProFlow, which combines the terminal MAP update and the interpolation step into a two-stage proximal guidance framework. At each homotopy time, we first refine the terminal state through a constrained proximal optimization that enforces the physics and the data, then interpolate back to an intermediate time along a flow-consistent bridge. The overall procedure is summarized in \Cref{alg:proflow}.

\begingroup
\paragraph{Terminal optimization step.}
\setlength{\abovedisplayskip}{6pt}\setlength{\belowdisplayskip}{6pt}\setlength{\belowdisplayshortskip}{6pt}
At homotopy time $t_n$, given the current state $\bm u_{t_n}$, we form a one-step terminal prediction via the flow velocity field $\bm v_\theta$ as
\[
\hat{\bm u}_1=\bm u_{t_n}+(1-t_n)\bm v_\theta(\bm u_{t_n},t_n).
\]
This plug-in estimate serves as the anchor for the constrained proximal problem
\[
\bm u_1\in\operatorname*{argmin}_{\bm u\in\mathcal{C}}\Bigl\{\|\bm u-\hat{\bm u}_1\|_2^2+\lambda\|\mathcal{H}[\bm u]-\bm y\|_2^2\Bigr\},
\]
where $\mathcal{H}$ is the measurement operator and $\bm y$ denotes the observed values on the observation domain $\mathcal{X}_{\mathcal{O}}\subseteq\Omega$.
\endgroup
The first term keeps $\bm u$ close to the flow prediction $\hat{\bm u}_1$, the second term enforces data fidelity, and the constraint $\bm u\in\mathcal{C}$ imposes $\mathcal{L}(\bm u)=\bm0$.
The parameter $\lambda$ controls the relative weight of the data misfit. While it can be interpreted as a surrogate for the variance ratio $\sigma_t^2/\sigma_\mathrm{obs}^2$ in the probabilistic model, we treat it as a tunable hyperparameter in practice.
For the generative noise injection, we adopt the schedule $\sigma_t=(1-t)/\sqrt{t^2+(1-t)^2}$ following $\Pi$GDM~\citep{song2023pseudoinverse}, and we assume a fixed observation noise level $\sigma_{\mathrm{obs}}=0.05$.

\begin{remark}[Classification as a hard constraint method]
\rightskip=.1pt
We classify ProFlow as a hard constraint method because the formulation in \eqr{eq:terminal_opt} explicitly restricts the search space to $\mathcal{C}$ via the indicator function implicit in the constrained minimization. In the implementation (\Cref{sec:proflow-details}), this sub-problem is solved by the penalized relaxation
\begin{align*}
L(\bm u_1)=\|\bm u_1-\hat{\bm u}_1\|_2^2+\lambda_{\mathrm{obs}}\|(\bm u_1-\bm c)\odot\bm m\|_2^2+\lambda_{\mathrm{pde}}\|\mathcal{R}(\bm u_1)\|_2^2,
\end{align*}
in which the physical constraint enters through the residual penalty $\lambda_{\mathrm{pde}}\|\mathcal{R}(\bm u_1)\|_2^2$ reduced by a fixed number of gradient steps ($K=3$). Physical consistency is thus enforced only approximately, at an accuracy set by $\lambda_{\mathrm{pde}}$, the step size, and $K$, and the exactly-constrained solve is recovered as $\lambda_{\mathrm{pde}}\to\infty$ and $K\to\infty$. This contrasts with soft guidance methods~\citep{ben2024d,huang2024diffusionpde} that add residual penalties to the outer sampling loop without a projection mechanism. The difference is structural in that ProFlow devotes a dedicated inner optimization to the terminal prediction $\hat{\bm u}_1$, and empirically attains substantially lower PDE residuals than these baselines (\Cref{sec:experiments}).\looseness=1
\end{remark}

\paragraph{Interpolation step.}
To keep the iterate near the learned flow manifold, we interpolate between a fresh noise realization and the refined terminal state. Given $\bm u_1$ and a draw $\bm\varepsilon\sim\mu_0$, we set the state for the next time step $t_{n+1}$ as
\[
\bm u_{t_{n+1}}=(1-t_{n+1})\bm\varepsilon+t_{n+1}\bm u_1.
\]
This update reinitializes the next state along a linear bridge consistent with the flow matching training procedure. In this way, the subsequent evaluation of the velocity field $\bm v_\theta(\cdot,t_{n+1})$ operates on samples that remain close to the learned model manifold.

\begin{algorithm}[tbp]
\caption{Proximal flow guidance (ProFlow)}\label{alg:proflow}
\fontsize{10}{12}\selectfont
\begin{algorithmic}[1]
\Require Pretrained FFM velocity $\bm v_\theta$; measurement operator $\mathcal{H}$; observations $\bm y$; constraint set $\mathcal{C}=\{\bm u\colon\mathcal{L}(\bm u)=\bm0\}$; schedule $0=t_0<\cdots<t_N=1$
\Ensure Generated sample $\bm u_{t_N}$

\State Sample $\bm u_{t_0}\sim\mu_0$
\For{$n=0\text{ to }N-1$}
\State $\hat{\bm u}_1\gets\bm u_{t_n}+(1-t_n)\bm v_\theta(\bm u_{t_n},t_n)$
\State Perform terminal optimization:
\[
\bm u_1\gets\operatorname*{argmin}_{\bm u\in\mathcal{C}}\Bigl\{\|\bm u-\hat{\bm u}_1\|_2^2+\lambda\|\mathcal{H}[\bm u]-\bm y\|_2^2\Bigr\}
\]
\State Sample $\bm\varepsilon\sim\mu_0$
\State $\bm u_{t_{n+1}}\gets(1-t_{n+1})\bm\varepsilon+t_{n+1}\bm u_1$
\EndFor
\State \Return $\bm u_{t_N}$
\end{algorithmic}
\end{algorithm}
\ProductionCompactSection{Experiments}\label{sec:experiments}
\subsection{PDE families and the experimental setup}
We consider three elliptic PDE families together with a time-dependent conservation law. The elliptic benchmarks are Poisson, Helmholtz, and Darcy equations on a bounded domain $\Omega\subset\mathbb{R}^2$ with Dirichlet boundary conditions and heterogeneous coefficients. In a generic form they can be written as
\[
\begin{cases}
-\nabla\cdot(a(\bm x)\nabla u(\bm x))=f(\bm x),&\text{in }\Omega,\\
u(\bm x)=0,&\text{on }\partial\Omega,
\end{cases}
\]
for Poisson and Darcy flow, with $a$ interpreted as a diffusivity or permeability field in the Darcy setting, and
\[
\begin{cases}
-\Delta u(\bm x)+\kappa u(\bm x)=f(\bm x),&\text{in }\Omega,\\
u(\bm x)=0,&\text{on }\partial\Omega,
\end{cases}
\]
for Helmholtz, where $\kappa>0$ is a constant reaction coefficient. For the time-dependent benchmark, we use the one-dimensional viscous Burgers' equation with periodic boundary conditions on a bounded domain $\Omega\subset\mathbb{R}$:
\[
\begin{cases}
\partial_t u+u\partial_x u-\nu\partial_x^2u=0,&(t,x)\in(0,T]\times\Omega,\\
u(0,x)=u_0(x),&x\in\Omega.
\end{cases}
\]
Here, the viscosity satisfies $\nu>0$. As detailed in \Cref{rem:time_dependent_extension}, we treat this problem as a global sampling task over the space-time domain $[0,T]\times\Omega$, where the target variable represents the entire spatiotemporal trajectory subject to the evolution constraint.

For the elliptic problems, we work with paired coefficient and solution fields $(\bm a,\bm u)$ discretized on a $128\times128$ grid. All elliptic datasets are obtained from the public releases accompanying Fourier Neural Operator (FNO)~\citep{li2021fourier} and DiffusionPDE~\citep{huang2024diffusionpde}. For Burgers' equation, we follow the DiffusionPDE~\citep{huang2024diffusionpde} setup and represent spatiotemporal fields on a regular space-time grid.

\paragraph{Data-generating random fields.}
The regularity and effective complexity of these benchmark inputs are governed by the spectral decay of the underlying Gaussian random fields. For the elliptic datasets, the zero-mean Gaussian random field $g$ on $[0,1]^2$ is sampled with covariance operator $\mathcal{C}_{\mathrm g}=(-\Delta+\tau^2 I)^{-\alpha}$ under zero-Neumann boundary conditions, with $\alpha=2$ and $\tau=3$. In the cosine (Karhunen--Lo\`eve) basis, the standard deviation of mode $(k_1,k_2)$ is proportional to $\tau^{\alpha-1}[\pi^2(k_1^2+k_2^2)+\tau^2]^{-\alpha/2}$, so the covariance eigenvalues decay at high frequencies as $[\pi^2(k_1^2+k_2^2)+9]^{-2}=\mathcal{O}(|k|^{-4})$. For the Poisson and Helmholtz benchmarks the sampled field is used as the random input/source field. For the Darcy benchmark the permeability is obtained from the same underlying field through the standard FNO thresholding map, namely $a(\bm x)=12$ where $g(\bm x)\geq0$ and $a(\bm x)=3$ where $g(\bm x)<0$, so that although the final coefficient is piecewise constant, its geometry and spatial scales are determined by a field whose covariance modes decay as $\mathcal{O}(|k|^{-4})$. For the viscous Burgers' benchmark, the initial condition is sampled from a periodic Gaussian random field with covariance $\mathcal{C}_{\mathrm B}=\sigma^2(-\Delta+\tau^2 I)^{-\gamma}$, where $\gamma=2.5$, $\tau=7$, and $\sigma=49$, whose Fourier-mode covariance decays as $[(2\pi k)^2+\tau^2]^{-\gamma}=\mathcal{O}(|k|^{-5})$.\looseness=1

\paragraph{Observation operator.}
In every experiment the measurement operator $\mathcal{H}$ is a linear sub-sampling (masking) operator. Let $\bm m$ be a binary mask flagging the observed entries and $\odot$ the entrywise product. Then $\mathcal{H}[\bm u]=\bm m\odot\bm u$ and $\bm y=\bm m\odot\bm c$, so the data term $\|\mathcal{H}[\bm u]-\bm y\|_2^2$ reduces to the masked form $\|(\bm u-\bm c)\odot\bm m\|_2^2$ used in the implementation, and it is a convex quadratic. The mask is what distinguishes the tasks, acting on the paired field $(\bm a,\bm u)$. For the elliptic \emph{forward} problem $\bm m$ reveals the entire coefficient channel $\bm a$ and hides $\bm u$, so $\mathcal{H}$ returns the full coefficient field. For the \emph{inverse} problem the roles are reversed and $\mathcal{H}$ returns the full solution field $\bm u$. For \emph{joint sparse reconstruction} $\bm m$ reveals a random $50\%$ of the grid points of both $\bm a$ and $\bm u$. For the Burgers' \emph{IC} configuration $\bm m$ selects the initial slice $\bm u_0$ together with sparse spatiotemporal samples at later times, and for the \emph{BC} configuration it selects the boundary traces together with sparse interior samples. Because $\mathcal{H}$ is linear in all cases, the observation term contributes a convex quadratic to the proximal step in \eqr{eq:terminal_opt}.

\paragraph{Generative prior.}
We employ Functional Flow Matching (FFM)~\citep{kerrigan2024functional} as the generative prior. For each PDE family, we train a separate FFM model on the paired training data $(\bm a,\bm u)$ using only full fields, with no access to test targets, observation masks, or inverse configurations. After training, the FFM prior is frozen and used for all downstream sampling tasks. The architectural details of the pretrained FFM are given in \Cref{sec:pretrained model-details}.

\paragraph{Baselines.}
We compare ProFlow with four constrained sampling baselines and include unconstrained FFM as a reference. A high-level comparison is summarized in \Cref{tab:comp}. The implementation details for all baselines can be found in \Cref{sec:details_baseline}, while those for ProFlow are provided in \Cref{sec:proflow-details}.
\begin{itemize}
\rightskip=.1pt
\item \emph{FFM}~\citep{kerrigan2024functional} denotes the standard Functional Flow Matching prior. It generates samples from the learned distribution without any sampling-time conditioning, serving as a reference for the unconstrained generative performance.
\item \emph{ECI}~\citep{cheng2025gradientfree} is a gradient-free framework designed for hard-constrained generation. It employs a three-stage iterative process: (i) an \emph{extrapolation} step that estimates the terminal state based on the current state; (ii) a \emph{correction} step that replaces part of the state to strictly enforce observation constraints; and (iii) an \emph{interpolation} step that maps the corrected state back to an intermediate time $t$ to maintain alignment with the generative probability path. However, because the framework relies on geometric projections to enforce constraints, it lacks a mechanism to ensure physical consistency during correction.
\item \emph{DiffusionPDE}~\citep{huang2024diffusionpde} is a soft guidance method. It employs gradient-based guidance, injecting a correction term calculated from the gradients of the PDE residual and data fidelity loss with respect to the current state estimate. This steers the trajectory toward physically feasible regions but, unlike projection methods, does not guarantee strict constraint satisfaction.
\item \emph{D-Flow}~\citep{ben2024d} casts the sampling problem as an optimization over the initial noise space. Unlike DiffusionPDE~\citep{huang2024diffusionpde}, which applies local guidance at each denoising step, D-Flow is a soft guidance framework that performs global optimization on the source distribution $\bm u_0$ by differentiating through the ODE solver. While the flow diffeomorphism implicitly regularizes the solution to remain close to the data manifold, the requirement to backpropagate gradients through the entire integration trajectory results in significant memory consumption and prolonged sampling time.
\item \emph{PCFM}~\citep{utkarsh2025physics} can be viewed as a generalization of ECI~\citep{cheng2025gradientfree} that enforces arbitrary nonlinear constraints. It operates by repeatedly shooting along the learned flow to the terminal time, applying a single Gauss--Newton update to bring the terminal state closer to the constraint manifold, and transporting this correction back to the current state via an optimal-transport displacement interpolant. However, incorporating PDE residuals into the Gauss--Newton step incurs prohibitive computational costs due to Jacobian evaluations. Furthermore, the projection-based update focuses solely on constraint satisfaction and does not guarantee adherence to the generative prior.
\end{itemize}

\paragraph{Evaluation metrics.}
Following prior work~\citep{cheng2025gradientfree,huang2024diffusionpde}, we report several metrics that quantify reconstruction accuracy, distributional statistics, and physical consistency.
\begin{ProductionMetricsList}
\item \emph{Reconstruction error} (RE) measures the mean squared discrepancy between reconstructed fields and ground truth fields. This serves as our primary metric for overall observational consistency and reconstruction quality. In the joint setting, RE is computed over both solution and coefficient fields.
\item \emph{Mean MSE} (MMSE) measures the mean squared error of the sample mean with respect to the ground truth field. This reflects the accuracy of the estimated posterior mean (the first moment of the distribution).
\item \emph{Standard deviation MSE} (SMSE) measures the mean squared error of the sample standard deviation. This quantifies how well the method captures predictive uncertainty and the true distributional statistics of the posterior.
\item \emph{PDE error} (PDE Err.) measures the violation of the governing equations encoded by the differential operator, directly quantifying the physical consistency of the generated samples. We follow~\citep{zhang2025physics} for computing the PDE error of the Poisson, Helmholtz, and Burgers' equations, and follow~\citep{huang2024diffusionpde} for Darcy flow.
\end{ProductionMetricsList}

\subsection{Elliptic PDE results}
We first study the three elliptic PDE families: Poisson, Helmholtz, and Darcy equations. Within each setting, we analyze performance across three distinct observation regimes, selected to represent standard challenges in data-constrained PDE solving: (i) \emph{forward problem}, where only the coefficient field $\bm a$ is observed and the goal is to infer the corresponding solution field $\bm u$; (ii) \emph{inverse problem}, where only the solution field $\bm u$ is available and the goal is to recover the coefficient field $\bm a$; and (iii) \emph{joint sparse reconstruction}, where both coefficient and solution fields $(\bm a,\bm u)$ are partially observed and the goal is to reconstruct both fields simultaneously.

\begin{figure}[!htbp]
\centering\captionsetup{skip=6.1594pt,belowskip=.179pt}
\includegraphics[width=.42\linewidth,viewport={0bp 0bp 125bp 125bp},clip]{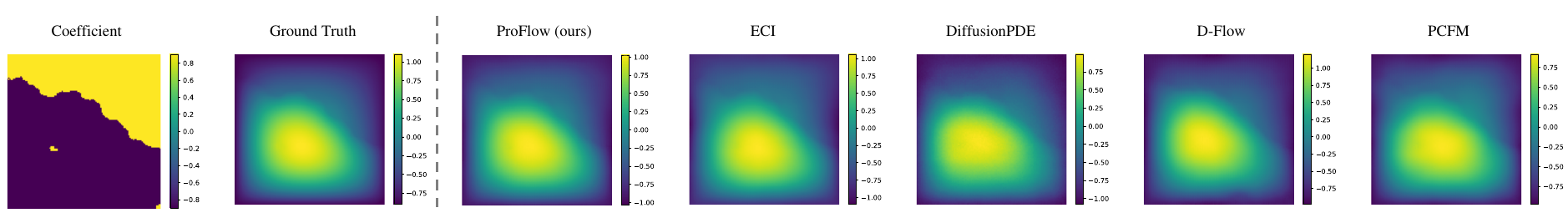}\hspace{3pt}%
\includegraphics[width=.42\linewidth,viewport={126bp 0bp 251bp 125bp},clip]{expr_forward-titles.pdf}\par\medskip
\includegraphics[width=.42\linewidth,viewport={254bp 0bp 383bp 125bp},clip]{expr_forward-titles.pdf}\hspace{3pt}%
\includegraphics[width=.42\linewidth,viewport={384bp 0bp 512bp 125bp},clip]{expr_forward-titles.pdf}\par\medskip
\includegraphics[width=.323\linewidth,viewport={513bp 0bp 642bp 125bp},clip]{expr_forward-titles.pdf}\hspace{3pt}%
\includegraphics[width=.323\linewidth,viewport={643bp 0bp 772bp 125bp},clip]{expr_forward-titles.pdf}\hspace{3pt}%
\includegraphics[width=.323\linewidth,viewport={773bp 0bp 901bp 125bp},clip]{expr_forward-titles.pdf}
\caption{Visual comparison of the forward problem on the Darcy equation. The first panel displays the input coefficient field. The remaining panels compare the ground truth solution against predictions from ProFlow and baseline approaches (ECI, DiffusionPDE, D-Flow, and PCFM). ProFlow predicts the solution field with high visual fidelity, capturing details that match the ground truth more closely than the baseline predictions do.}\label{fig:expr_forward}
\end{figure}

\begin{table}[tbp]
\centering
\renewcommand{\arraystretch}{1.04}\captionsetup{skip=5.3542pt}\caption{Forward and inverse problem results for Poisson, Helmholtz, and Darcy equations. In the forward problem setting the coefficient field $\bm a$ is fully observed and the task is to infer the corresponding solution field $\bm u$. 
In the inverse problem setting $\bm u$ is available and the task is to recover $\bm a$. Reported quantities include reconstruction error (RE), mean MSE (MMSE), standard deviation MSE (SMSE), and PDE error (PDE Err.), averaged over $500$ samples. Lower values indicate better performance. The best-performing method in each row is shown in boldface and the second best is underlined. PCFM results are omitted (indicated by ``--'') in configurations where the method exhibited numerical divergence, despite extensive hyperparameter tuning.}\label{tab:forward_inverse_unified}
\setlength{\tabcolsep}{2.08pt}
\begin{cicptabular}{l|l|l|c|c|c|c|c}\hline
{\footnotesize PDE} & {\footnotesize Setting} & {\footnotesize Metric} & {\footnotesize ProFlow} & {\footnotesize ECI~\citep{cheng2025gradientfree}} & {\footnotesize DiffusionPDE~\citep{huang2024diffusionpde}} & {\footnotesize D-Flow~\citep{ben2024d}} & {\footnotesize PCFM~\citep{utkarsh2025physics}}\\\hline
\multirow{8}{*}{Poisson} & \multirow{4}{*}{Fwd.} & RE ($\times10^{-2}$)&\textbf{0.67}&5.05&1.60&\underline{1.24}&2.18\\
 &  & MMSE ($\times10^{-3}$)&\textbf{3.07}&9.39&7.50&\underline{4.08}&20.6\\
 &  & SMSE ($\times10^{-3}$)&\textbf{1.58}&22.0&\underline{8.07}&8.97&5.44\\
 &  & PDE Err. ($\times10^{-9}$)&\textbf{1.17}&6.60&8.79&6.97&\underline{2.62}\\\cline{2-8}
 & \multirow{4}{*}{Inv.} & RE ($\times10^{-2}$)&\underline{3.40}&5.25&6.32&\textbf{2.96}&--\\
 &  & MMSE ($\times10^{-2}$)&\underline{1.58}&2.31&2.02&\textbf{0.14}&--\\
 &  & SMSE ($\times10^{-4}$)&\textbf{2.39}&\underline{6.84}&33.6&49.4&--\\
 &  & PDE Err. ($\times10^{-6}$)&\textbf{1.39}&1.90&6.92&\underline{1.52}&--\\\hline
\multirow{8}{*}{Helmholtz} & \multirow{4}{*}{Fwd.} & RE ($\times10^{-3}$)&\textbf{2.46}&\underline{8.53}&130&19.3&244\\
 &  & MMSE ($\times10^{-4}$)&\textbf{6.55}&67.0&$1.03\times10^3$&\underline{33.7}&22.5\\
 &  & SMSE ($\times10^{-3}$)&\textbf{1.00}&8.71&\underline{7.07}&8.69&7.05\\
 &  & PDE Err. ($\times10^{-9}$)&\textbf{5.95}&\underline{6.45}&9.19&6.89&8.67\\\cline{2-8}
 & \multirow{4}{*}{Inv.} & RE ($\times10^{-1}$)&\underline{1.84}&$5.70\times10^3$&3.63&\textbf{1.64}&9.96\\
 &  & MMSE ($\times10^{-2}$)&\underline{2.74}&$6.30\times10^3$&3.01&\textbf{0.40}&34.2\\
 &  & SMSE ($\times10^{-2}$)&\textbf{2.06}&$2.16\times10^4$&\underline{2.36}&\underline{2.36}&87.2\\
 &  & PDE Err. ($\times10^{-6}$)&\underline{7.82}&150&10.8&\textbf{4.37}&17.8\\\hline
\multirow{8}{*}{Darcy} & \multirow{4}{*}{Fwd.} & RE ($\times10^{-3}$)&\textbf{3.10}&\underline{5.72}&9.42&6.04&3.13\\
 &  & MMSE ($\times10^{-3}$)&\textbf{1.41}&\underline{2.11}&4.82&10.83&2.26\\
 &  & SMSE ($\times10^{-3}$)&\textbf{1.07}&\underline{2.07}&2.08&8.37&5.23\\
 &  & PDE Err. ($\times10^{-7}$)&3.03&3.07&2.78&\underline{2.46}&\textbf{1.16}\\\cline{2-8}
 & \multirow{4}{*}{Inv.} & RE ($\times10^{-1}$)&\textbf{2.41}&\underline{3.23}&3.92&4.32&10.7\\
 &  & MMSE ($\times10^{-2}$)&\textbf{1.38}&\underline{2.03}&5.07&2.51&79.7\\
 &  & SMSE ($\times10^{-3}$)&\textbf{3.56}&7.19&5.19&\underline{3.70}&183\\
 &  & PDE Err. ($\times10^{-5}$)&\textbf{1.17}&\textbf{1.17}&2.84&\underline{1.44}&3.26\\\hline
\end{cicptabular}
\end{table}

\paragraph{Forward and inverse problems.}
In the forward problem setting we observe the full coefficient field $\bm a$ and aim to reconstruct the associated solution field $\bm u$. As reported in \Cref{tab:forward_inverse_unified} and illustrated in \cref{fig:expr_forward}, ProFlow achieves the best reconstruction quality across all three elliptic PDE families. For the Poisson equation, ProFlow attains the lowest RE, MMSE, and SMSE, with RE reduced by approximately a factor of two relative to the next-best method (D-Flow) and by nearly an order of magnitude relative to ECI. A similar pattern holds for Helmholtz and Darcy forward solves: ProFlow consistently yields the smallest RE, MMSE, and SMSE, while the alternatives incur noticeably larger reconstruction errors (for example, DiffusionPDE is two to three orders of magnitude worse on Helmholtz). In terms of PDE residuals, ProFlow achieves the lowest values for Poisson and Helmholtz and remains competitive on Darcy, where the residuals of all methods are of the same order of magnitude. These results indicate that, in the forward problem setting, ProFlow provides accurate reconstruction and distributional statistics, and solutions that closely satisfy the governing equations.

In the inverse problem setting we observe the solution field $\bm u$ and aim to recover the underlying coefficient field $\bm a$. As reported in \Cref{tab:forward_inverse_unified}, for Poisson, D-Flow achieves the smallest RE and MMSE, while ProFlow attains the best SMSE and the lowest PDE residual. Thus, D-Flow is slightly more accurate in terms of point estimates of the coefficients, whereas ProFlow offers more accurate distributional statistics and closer adherence to the PDE. For Helmholtz inverse problems, ProFlow and D-Flow clearly dominate the other baselines: ECI and DiffusionPDE suffer from very large errors, while ProFlow provides the best SMSE and D-Flow attains the lowest RE, MMSE, and PDE residual. This suggests that D-Flow is particularly effective at shrinking reconstruction error for this more challenging operator, whereas ProFlow gives the most stable distributional estimates. On Darcy inverse problems, ProFlow achieves the lowest RE, MMSE, and SMSE, and shares the best PDE residual with ECI, with D-Flow close behind. Overall, across the three elliptic PDEs, ProFlow either matches or surpasses competing methods on most metrics.

\paragraph{Joint sparse reconstruction.}
In the joint sparse reconstruction setting we assume that both the coefficient field $\bm a$ and the solution field $\bm u$ are only partially observed, with $50\%$ of their spatial grid points revealed in our experiments, and seek to reconstruct the full pair $(\bm a,\bm u)$ simultaneously. The quantitative results in \Cref{tab:joint_pde} show that ProFlow provides the most balanced performance across the three elliptic PDE families. For Poisson, ProFlow achieves the lowest reconstruction error and MMSE, with RE $6.93\times10^{-2}$ versus $7.00\times10^{-2}$ for DiffusionPDE and substantially larger errors for ECI and D-Flow. It also attains the smallest PDE residual, around $1.76\times10^{-8}$, while ECI incurs residuals several orders of magnitude larger. ECI obtains the best SMSE, with ProFlow a close second, so ProFlow trades a modest increase in variance error for markedly better reconstruction accuracy and physical consistency. For Helmholtz, ProFlow and DiffusionPDE are the only methods that achieve low reconstruction error in the joint setting. ProFlow yields the smallest RE and PDE residual, whereas DiffusionPDE attains the best MMSE but with slightly higher RE. ECI again produces the lowest SMSE but at the cost of very large reconstruction and PDE errors. For Darcy flow, ProFlow clearly dominates the baselines, achieving the lowest RE, MMSE, SMSE, and PDE residual, with RE roughly four to six times smaller than the competing methods and PDE error reduced by more than a factor of two relative to DiffusionPDE. Collectively, these results indicate that ProFlow enables accurate and physically consistent reconstructions under severe undersampling and either matches or improves upon the best competing method on most metrics while keeping PDE residuals small.\looseness=1

\begin{table}[tbp]
\centering
\renewcommand{\arraystretch}{1.04}\captionsetup{skip=5.3542pt}\caption{Joint sparse reconstruction results for Poisson, Helmholtz, and Darcy equations. Both the coefficient field $\bm a$ and the solution field $\bm u$ are reconstructed from sparse observations of both fields (in our experiments $50\%$ of the spatial grid points of $\bm a$ and $\bm u$ are revealed). Reported metrics are reconstruction error (RE), mean MSE (MMSE), standard deviation MSE (SMSE), and PDE error (PDE Err.), averaged over $500$ samples. Lower values indicate better performance. The best value in each row is shown in boldface and the second best is underlined. PCFM results are omitted (indicated by ``--'') in configurations where the method exhibited numerical divergence, despite extensive hyperparameter tuning.}\label{tab:joint_pde}
\setlength{\tabcolsep}{2.98pt}
\begin{cicptabular}{l|l|c|c|c|c|c}\hline
PDE & Metric & ProFlow & ECI~\citep{cheng2025gradientfree} & DiffusionPDE~\citep{huang2024diffusionpde} & D-Flow~\citep{ben2024d} & PCFM~\citep{utkarsh2025physics}\\\hline
\multirow{4}{*}{Poisson} & RE ($\times10^{-2}$)&\textbf{6.93}&14.6&\underline{7.00}&73.1&--\\
 & MMSE ($\times10^{-2}$)&\textbf{1.19}&2.20&\underline{1.26}&12.8&--\\
 & SMSE ($\times10^{-1}$)&\underline{1.61}&\textbf{0.84}&1.62&8.26&--\\
 & PDE Err. ($\times10^{-8}$)&\textbf{1.76}&$6.25\times10^3$&\underline{5.48}&9.03&--\\\hline
\multirow{4}{*}{Helmholtz} & RE ($\times10^{-2}$)&\textbf{5.94}&12.6&\underline{5.97}&95.2&88.1\\
 & MMSE ($\times10^{-3}$)&\underline{6.01}&13.5&\textbf{5.31}&188&332\\
 & SMSE ($\times10^{-1}$)&\underline{1.42}&\textbf{0.76}&\underline{1.42}&11.5&3.58\\
 & PDE Err. ($\times10^{-8}$)&\textbf{7.32}&$5.84\times10^3$&\underline{10.5}&46.1&$8.39\times10^3$\\\hline
\multirow{4}{*}{Darcy} & RE ($\times10^{-2}$)&\textbf{5.34}&30.1&\underline{24.1}&46.3&22.1\\
 & MMSE ($\times10^{-2}$)&\textbf{1.61}&\underline{6.47}&6.69&14.6&13.4\\
 & SMSE ($\times10^{-3}$)&\textbf{4.92}&63.6&57.7&\underline{11.3}&94.3\\
 & PDE Err. ($\times10^{-5}$)&\textbf{2.74}&55.7&\underline{6.83}&10.7&28.0\\\hline
\end{cicptabular}
\end{table}

\ProductionCompactSubsection{Controllable generation with initial and boundary conditions}
We next evaluate ProFlow on controllable generation of time-dependent PDE solutions under prescribed initial or boundary conditions. We focus on the one-dimensional viscous Burgers' equation with given viscosity $\nu>0$ and prescribed initial and boundary data. The goal is to generate the full temporal trajectory $\bm u_{0:T}$ from partial information. We consider two configurations: (i) initial condition (IC) configuration, where the initial state $\bm u_0$ is known and sparse spatial observations of $\bm u(t,\cdot)$ at later times are provided; the task is to reconstruct the full space-time evolution consistent with the Burgers' dynamics and the observed values; and (ii) boundary condition (BC) configuration, where the temporal traces of the boundary values are prescribed together with sparse interior observations and the task is again to reconstruct the full trajectory.

\Cref{tab:burgers_bc_ic} summarizes the quantitative results for BC and IC configurations, while \cref{fig:expr_icbc} visualizes the reconstructed trajectories across methods. Under the BC configuration, ProFlow attains the lowest RE, MMSE, and SMSE. Notably, ProFlow reduces reconstruction error by nearly 40\% relative to the second-best method, DiffusionPDE, and by an order of magnitude compared to ECI and D-Flow. Regarding physical consistency, the PDE error of ProFlow remains competitive with D-Flow and ECI and is significantly lower than that of DiffusionPDE and PCFM. This indicates that ProFlow recovers trajectories that are both highly accurate and physically valid. In the IC configuration, ProFlow again demonstrates strong performance, achieving the lowest MMSE and ranking second in RE and SMSE, closely trailing D-Flow. ProFlow and DiffusionPDE have the same PDE residual, which is comparable to that of the best-performing method, D-Flow, and significantly lower than that of ECI. As shown qualitatively in \cref{fig:expr_icbc}, ProFlow effectively incorporates partial temporal information---whether as initial states or boundary traces---to generate sharp, physically consistent solutions where other baselines may exhibit over-smoothing or artifacts.

\begin{table}[tbp]
\centering
\renewcommand{\arraystretch}{.9988}\captionsetup{skip=5.3542pt}\caption{Burgers' reconstruction under boundary condition (BC) and initial condition (IC) configurations. For BC, we are given boundary observations over time, whereas for IC, we are given the initial snapshot of the system. Reported metrics are reconstruction error (RE), mean MSE (MMSE), standard deviation MSE (SMSE), and PDE error (PDE Err.). Lower values indicate better performance. The best value in each row is shown in boldface and the second best is underlined. PCFM results are omitted (indicated by ``--'') in configurations where the method exhibited numerical divergence, despite extensive hyperparameter tuning.}\label{tab:burgers_bc_ic}
\setlength{\tabcolsep}{3.55pt}
\begin{cicptabular}{l|l|c|c|c|c|c}\hline
Condition & Metric & ProFlow & ECI~\citep{cheng2025gradientfree} & DiffusionPDE~\citep{huang2024diffusionpde} & D-Flow~\citep{ben2024d} & PCFM~\citep{utkarsh2025physics}\\\hline
\multirow{4}{*}{IC} & RE ($\times10^{-2}$)&\underline{1.71}&3.68&5.80&\textbf{1.40}&--\\
 & MMSE ($\times10^{-2}$)&\textbf{1.95}&3.36&2.44&\underline{2.46}&--\\
 & SMSE ($\times10^{-3}$)&\underline{9.87}&13.12&33.75&\textbf{8.82}&--\\
 & PDE Err. ($\times10^{-4}$)&\underline{1.67}&2.19&\underline{1.67}&\textbf{1.62}&--\\\hline
\multirow{4}{*}{BC} & RE ($\times10^{-3}$)&\textbf{0.32}&4.52&\underline{0.53}&3.15&178.99\\
 & MMSE ($\times10^{-4}$)&\textbf{2.90}&26.3&\underline{5.14}&42.84&154\\
 & SMSE ($\times10^{-4}$)&\textbf{4.30}&44.0&11.89&198.03&49.5\\
 & PDE Err. ($\times10^{-4}$)&3.15&\textbf{2.44}&58.52&\underline{2.50}&12.4\\\hline
\end{cicptabular}
\end{table}
\begin{figure}[!p]
\centering\begin{minipage}[c][\dimexpr\textheight-1pt\relax][s]{\linewidth}
\centering
\noindent\makebox[0pt][l]{\color{gray}\setlength{\unitlength}{1pt}\begin{picture}(0,0)\multiput(213,120)(0,-8){59}{\line(0,-1){4}}\end{picture}}%
\begin{minipage}[t]{.49\linewidth}\centering
\begin{overpic}[pagebox=cropbox,width=130pt,trim={0bp 313.7188bp 464.1871bp 72.729bp},clip]{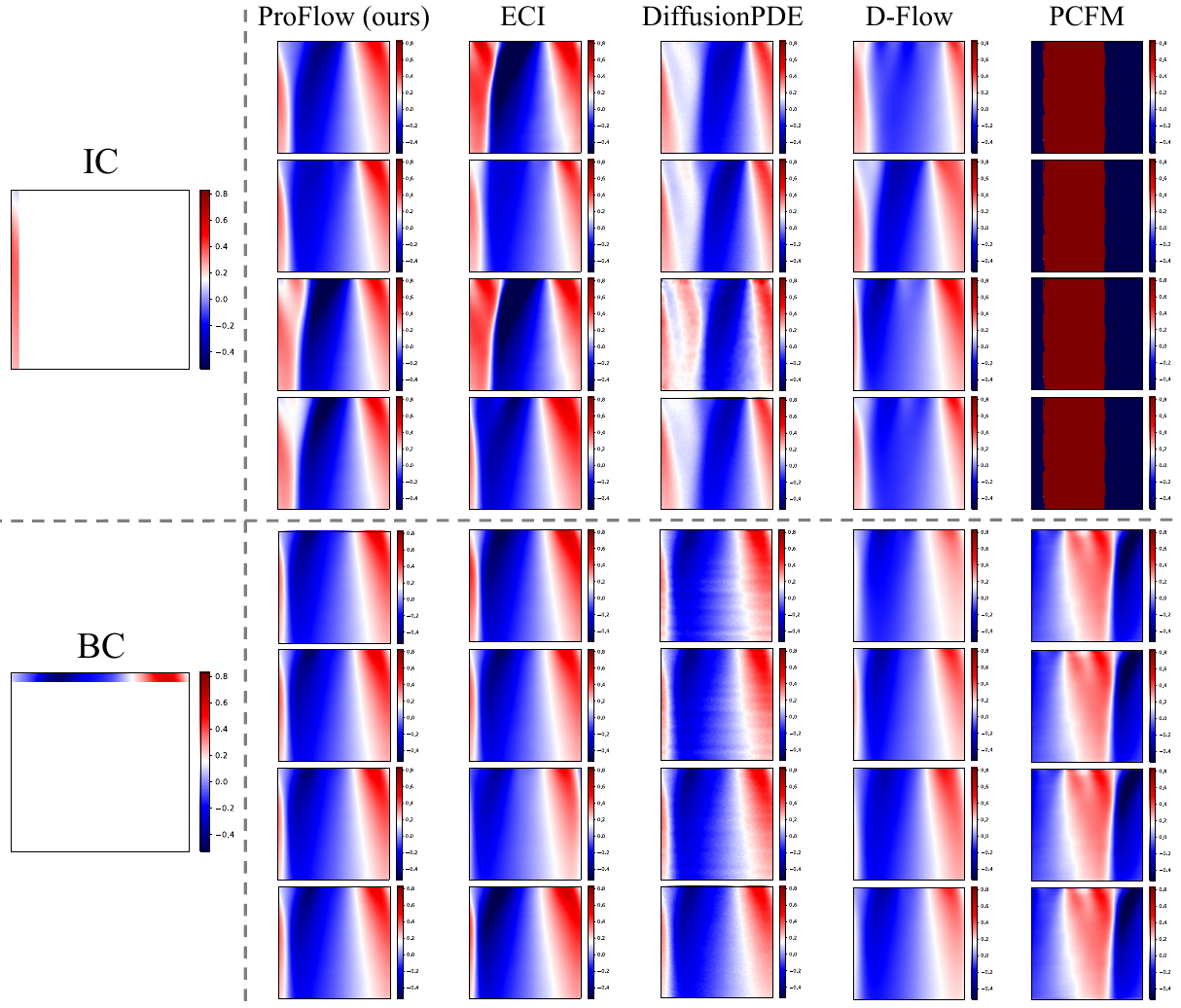}
\put(0,77.079){\color{white}\rule{130pt}{21.095pt}}
\put(45,83.165){\makebox(0,0){\fontsize{9.5}{10}\selectfont IC}}
\put(89.8,5.3753){\color{white}\rule{14pt}{99pt}}
\put(99.4,74.1379){\makebox(0,0)[r]{\fontsize{6.5}{6.5}\selectfont 0.8}}
\put(99.4,63.1846){\makebox(0,0)[r]{\fontsize{6.5}{6.5}\selectfont 0.6}}
\put(99.4,52.2312){\makebox(0,0)[r]{\fontsize{6.5}{6.5}\selectfont 0.4}}
\put(99.4,41.1765){\makebox(0,0)[r]{\fontsize{6.5}{6.5}\selectfont 0.2}}
\put(99.4,30.2231){\makebox(0,0)[r]{\fontsize{6.5}{6.5}\selectfont 0.0}}
\put(99.4,19.1684){\makebox(0,0)[r]{\fontsize{6.5}{6.5}\selectfont -0.2}}
\put(99.4,8.2150){\makebox(0,0)[r]{\fontsize{6.5}{6.5}\selectfont -0.4}}
\end{overpic}\end{minipage}\hfill
\begin{minipage}[t]{.49\linewidth}\centering
\begin{overpic}[pagebox=cropbox,width=130pt,trim={0bp 75.6094bp 464.1871bp 311.7986bp},clip]{expr_icbc.pdf}
\put(0,77.079){\color{white}\rule{130pt}{21.095pt}}
\put(45,83.165){\makebox(0,0){\fontsize{9.5}{10}\selectfont BC}}
\put(89.8,5.3753){\color{white}\rule{14pt}{99pt}}
\put(99.4,74.2394){\makebox(0,0)[r]{\fontsize{6.5}{6.5}\selectfont 0.8}}
\put(99.4,63.1846){\makebox(0,0)[r]{\fontsize{6.5}{6.5}\selectfont 0.6}}
\put(99.4,52.2312){\makebox(0,0)[r]{\fontsize{6.5}{6.5}\selectfont 0.4}}
\put(99.4,41.1765){\makebox(0,0)[r]{\fontsize{6.5}{6.5}\selectfont 0.2}}
\put(99.4,30.2231){\makebox(0,0)[r]{\fontsize{6.5}{6.5}\selectfont 0.0}}
\put(99.4,19.2698){\makebox(0,0)[r]{\fontsize{6.5}{6.5}\selectfont -0.2}}
\put(99.4,8.2150){\makebox(0,0)[r]{\fontsize{6.5}{6.5}\selectfont -0.4}}
\end{overpic}\end{minipage}\par\vspace{3pt}\noindent{\color{gray}\hbox to\linewidth{\leaders\hbox{\kern2pt\rule{4pt}{.4pt}\kern2pt}\hfill}}\par\vfill
\begin{minipage}[t]{.49\linewidth}\centering
{\sffamily\fontsize{8}{9}\selectfont ProFlow (ours)}\par\vspace{2pt}
\includegraphics[pagebox=cropbox,width=.242\linewidth,trim={135.8479bp 420.7721bp 375.6219bp 19.2024bp},clip]{expr_icbc.pdf}\hspace{1pt}%
\includegraphics[pagebox=cropbox,width=.242\linewidth,trim={135.8479bp 361.9648bp 375.6219bp 77.7696bp},clip]{expr_icbc.pdf}\hspace{1pt}%
\includegraphics[pagebox=cropbox,width=.242\linewidth,trim={135.8479bp 303.3976bp 375.6219bp 136.3369bp},clip]{expr_icbc.pdf}\hspace{1pt}%
\includegraphics[pagebox=cropbox,width=.242\linewidth,trim={135.8479bp 244.8303bp 375.6219bp 194.9041bp},clip]{expr_icbc.pdf}%
\end{minipage}\hfill
\begin{minipage}[t]{.49\linewidth}\centering
{\sffamily\fontsize{8}{9}\selectfont ProFlow (ours)}\par\vspace{2pt}
\includegraphics[pagebox=cropbox,width=.242\linewidth,trim={135.8479bp 178.8221bp 375.6219bp 260.9123bp},clip]{expr_icbc.pdf}\hspace{1pt}%
\includegraphics[pagebox=cropbox,width=.242\linewidth,trim={135.8479bp 120.4949bp 375.6219bp 319.4795bp},clip]{expr_icbc.pdf}\hspace{1pt}%
\includegraphics[pagebox=cropbox,width=.242\linewidth,trim={135.8479bp 61.9277bp 375.6219bp 377.8068bp},clip]{expr_icbc.pdf}\hspace{1pt}%
\includegraphics[pagebox=cropbox,width=.242\linewidth,trim={135.8479bp 3.6004bp 375.6219bp 436.374bp},clip]{expr_icbc.pdf}%
\end{minipage}\par\vfill
\begin{minipage}[t]{.49\linewidth}\centering
{\sffamily\fontsize{8}{9}\selectfont ECI}\par\vspace{2pt}
\includegraphics[pagebox=cropbox,width=.242\linewidth,trim={230.6535bp 420.7721bp 281.0564bp 19.2024bp},clip]{expr_icbc.pdf}\hspace{1pt}%
\includegraphics[pagebox=cropbox,width=.242\linewidth,trim={230.6535bp 361.9648bp 281.0564bp 77.7696bp},clip]{expr_icbc.pdf}\hspace{1pt}%
\includegraphics[pagebox=cropbox,width=.242\linewidth,trim={230.6535bp 303.3976bp 281.0564bp 136.3369bp},clip]{expr_icbc.pdf}\hspace{1pt}%
\includegraphics[pagebox=cropbox,width=.242\linewidth,trim={230.6535bp 244.8303bp 281.0564bp 194.9041bp},clip]{expr_icbc.pdf}%
\end{minipage}\hfill
\begin{minipage}[t]{.49\linewidth}\centering
{\sffamily\fontsize{8}{9}\selectfont ECI}\par\vspace{2pt}
\includegraphics[pagebox=cropbox,width=.242\linewidth,trim={230.6535bp 178.8221bp 281.0564bp 260.9123bp},clip]{expr_icbc.pdf}\hspace{1pt}%
\includegraphics[pagebox=cropbox,width=.242\linewidth,trim={230.6535bp 120.4949bp 281.0564bp 319.4795bp},clip]{expr_icbc.pdf}\hspace{1pt}%
\includegraphics[pagebox=cropbox,width=.242\linewidth,trim={230.6535bp 61.9277bp 281.0564bp 377.8068bp},clip]{expr_icbc.pdf}\hspace{1pt}%
\includegraphics[pagebox=cropbox,width=.242\linewidth,trim={230.6535bp 3.6004bp 281.0564bp 436.374bp},clip]{expr_icbc.pdf}%
\end{minipage}\par\vfill
\begin{minipage}[t]{.49\linewidth}\centering
{\sffamily\fontsize{8}{9}\selectfont DiffusionPDE}\par\vspace{2pt}
\includegraphics[pagebox=cropbox,width=.242\linewidth,trim={325.219bp 420.7721bp 186.4909bp 19.2024bp},clip]{expr_icbc.pdf}\hspace{1pt}%
\includegraphics[pagebox=cropbox,width=.242\linewidth,trim={325.219bp 361.9648bp 186.4909bp 77.7696bp},clip]{expr_icbc.pdf}\hspace{1pt}%
\includegraphics[pagebox=cropbox,width=.242\linewidth,trim={325.219bp 303.3976bp 186.4909bp 136.3369bp},clip]{expr_icbc.pdf}\hspace{1pt}%
\includegraphics[pagebox=cropbox,width=.242\linewidth,trim={325.219bp 244.8303bp 186.4909bp 194.9041bp},clip]{expr_icbc.pdf}%
\end{minipage}\hfill
\begin{minipage}[t]{.49\linewidth}\centering
{\sffamily\fontsize{8}{9}\selectfont DiffusionPDE}\par\vspace{2pt}
\includegraphics[pagebox=cropbox,width=.242\linewidth,trim={325.219bp 178.8221bp 186.4909bp 260.9123bp},clip]{expr_icbc.pdf}\hspace{1pt}%
\includegraphics[pagebox=cropbox,width=.242\linewidth,trim={325.219bp 120.4949bp 186.4909bp 319.4795bp},clip]{expr_icbc.pdf}\hspace{1pt}%
\includegraphics[pagebox=cropbox,width=.242\linewidth,trim={325.219bp 61.9277bp 186.4909bp 377.8068bp},clip]{expr_icbc.pdf}\hspace{1pt}%
\includegraphics[pagebox=cropbox,width=.242\linewidth,trim={325.219bp 3.6004bp 186.4909bp 436.374bp},clip]{expr_icbc.pdf}%
\end{minipage}\par\vfill
\begin{minipage}[t]{.49\linewidth}\centering
{\sffamily\fontsize{8}{9}\selectfont D-Flow}\par\vspace{2pt}
\includegraphics[pagebox=cropbox,width=.242\linewidth,trim={419.7845bp 420.7721bp 91.9254bp 19.2024bp},clip]{expr_icbc.pdf}\hspace{1pt}%
\includegraphics[pagebox=cropbox,width=.242\linewidth,trim={419.7845bp 361.9648bp 91.9254bp 77.7696bp},clip]{expr_icbc.pdf}\hspace{1pt}%
\includegraphics[pagebox=cropbox,width=.242\linewidth,trim={419.7845bp 303.3976bp 91.9254bp 136.3369bp},clip]{expr_icbc.pdf}\hspace{1pt}%
\includegraphics[pagebox=cropbox,width=.242\linewidth,trim={419.7845bp 244.8303bp 91.9254bp 194.9041bp},clip]{expr_icbc.pdf}%
\end{minipage}\hfill
\begin{minipage}[t]{.49\linewidth}\centering
{\sffamily\fontsize{8}{9}\selectfont D-Flow}\par\vspace{2pt}
\includegraphics[pagebox=cropbox,width=.242\linewidth,trim={419.7845bp 178.8221bp 91.9254bp 260.9123bp},clip]{expr_icbc.pdf}\hspace{1pt}%
\includegraphics[pagebox=cropbox,width=.242\linewidth,trim={419.7845bp 120.4949bp 91.9254bp 319.4795bp},clip]{expr_icbc.pdf}\hspace{1pt}%
\includegraphics[pagebox=cropbox,width=.242\linewidth,trim={419.7845bp 61.9277bp 91.9254bp 377.8068bp},clip]{expr_icbc.pdf}\hspace{1pt}%
\includegraphics[pagebox=cropbox,width=.242\linewidth,trim={419.7845bp 3.6004bp 91.9254bp 436.374bp},clip]{expr_icbc.pdf}%
\end{minipage}\par\vfill
\begin{minipage}[t]{.49\linewidth}\centering
{\sffamily\fontsize{8}{9}\selectfont PCFM}\par\vspace{2pt}
\includegraphics[pagebox=cropbox,width=.242\linewidth,trim={507.6296bp 420.7721bp 4.0802bp 19.2024bp},clip]{expr_icbc.pdf}\hspace{1pt}%
\includegraphics[pagebox=cropbox,width=.242\linewidth,trim={507.6296bp 361.9648bp 4.0802bp 77.7696bp},clip]{expr_icbc.pdf}\hspace{1pt}%
\includegraphics[pagebox=cropbox,width=.242\linewidth,trim={507.6296bp 303.3976bp 4.0802bp 136.3369bp},clip]{expr_icbc.pdf}\hspace{1pt}%
\includegraphics[pagebox=cropbox,width=.242\linewidth,trim={507.6296bp 244.8303bp 4.0802bp 194.9041bp},clip]{expr_icbc.pdf}%
\end{minipage}\hfill
\begin{minipage}[t]{.49\linewidth}\centering
{\sffamily\fontsize{8}{9}\selectfont PCFM}\par\vspace{2pt}
\includegraphics[pagebox=cropbox,width=.242\linewidth,trim={507.6296bp 178.8221bp 4.0802bp 260.9123bp},clip]{expr_icbc.pdf}\hspace{1pt}%
\includegraphics[pagebox=cropbox,width=.242\linewidth,trim={507.6296bp 120.4949bp 4.0802bp 319.4795bp},clip]{expr_icbc.pdf}\hspace{1pt}%
\includegraphics[pagebox=cropbox,width=.242\linewidth,trim={507.6296bp 61.9277bp 4.0802bp 377.8068bp},clip]{expr_icbc.pdf}\hspace{1pt}%
\includegraphics[pagebox=cropbox,width=.242\linewidth,trim={507.6296bp 3.6004bp 4.0802bp 436.374bp},clip]{expr_icbc.pdf}%
\end{minipage}\par\vfill
\captionsetup{skip=7pt}\caption{Visual comparison of generated solutions for Burgers' equation. The left group displays results for the Initial Condition (IC) configuration, while the right group displays the Boundary Condition (BC) configuration. The top of each group visualizes the sparse input provided to the models. Since the provided IC/BC data does not uniquely constrain the PDE solution, the reconstruction is ill-posed. To capture this uncertainty, we vary the initial random noise to generate four distinct samples for each method. The samples are displayed from left to right in each row. While baselines like PCFM exhibit mode collapse (producing identical samples) or artifacts, ProFlow generates diverse, physically plausible trajectories consistent with the governing PDE.}\label{fig:expr_icbc}
\end{minipage}\end{figure}
\begin{figure}[!t]
\centering
\begin{overpic}[pagebox=cropbox,width=.323\linewidth,trim={0bp 0bp 778.0412bp 0bp},clip]{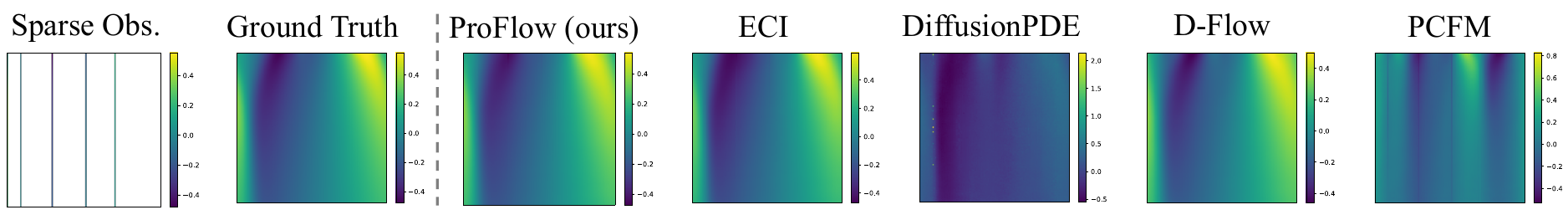}
\put(0,79.0196){\color{white}\rule{.323\linewidth}{0.069667\linewidth}}
\put(50,89.8039){\makebox(0,0){\fontsize{9.5}{10}\selectfont Sparse Obs.}}
\end{overpic}\hspace{3pt}%
\begin{overpic}[pagebox=cropbox,width=.323\linewidth,trim={122.4317bp 0bp 650.0881bp 0bp},clip]{expr_sparse_obs.pdf}
\put(0,75.6098){\color{white}\rule{.323\linewidth}{0.06666\linewidth}}
\put(50,85.9287){\makebox(0,0){\fontsize{9.5}{10}\selectfont Ground Truth}}
\end{overpic}\hspace{3pt}%
\begin{overpic}[pagebox=cropbox,width=.323\linewidth,trim={258.0667bp 0bp 514.4531bp 0bp},clip]{expr_sparse_obs.pdf}
\put(0,75.6098){\color{white}\rule{.323\linewidth}{0.06666\linewidth}}
\put(50,85.9287){\makebox(0,0){\fontsize{9.5}{10}\selectfont ProFlow (ours)}}
\end{overpic}\par\medskip
\begin{overpic}[pagebox=cropbox,width=.242\linewidth,trim={386.0198bp 0bp 388.6605bp 0bp},clip]{expr_sparse_obs.pdf}
\put(0,76.9084){\color{white}\rule{.242\linewidth}{0.050802\linewidth}}
\put(50,87.4046){\makebox(0,0){\fontsize{7}{10}\selectfont ECI}}
\end{overpic}\hspace{3pt}%
\begin{overpic}[pagebox=cropbox,width=.242\linewidth,trim={511.8124bp 0bp 252.3053bp 0bp},clip]{expr_sparse_obs.pdf}
\put(0,70.9507){\color{white}\rule{.242\linewidth}{0.046866\linewidth}}
\put(50,80.6338){\makebox(0,0){\fontsize{7}{10}\selectfont DiffusionPDE}}
\end{overpic}\hspace{3pt}%
\begin{overpic}[pagebox=cropbox,width=.242\linewidth,trim={648.1676bp 0bp 122.6717bp 0bp},clip]{expr_sparse_obs.pdf}
\put(0,74.6296){\color{white}\rule{.242\linewidth}{0.049296\linewidth}}
\put(50,84.8148){\makebox(0,0){\fontsize{7}{10}\selectfont D-Flow}}
\end{overpic}\hspace{3pt}%
\begin{overpic}[pagebox=cropbox,width=.242\linewidth,trim={777.8012bp 0bp 0bp 0bp},clip]{expr_sparse_obs.pdf}
\put(0,78.865){\color{white}\rule{.242\linewidth}{0.052094\linewidth}}
\put(50,89.6282){\makebox(0,0){\fontsize{7}{10}\selectfont PCFM}}
\end{overpic}
\captionsetup{skip=7pt}\caption{Visual comparison of sparse-in-time reconstruction on Burgers' equation. The first panel visualizes the sparse input data, consisting of observations at only five randomly sampled timestamps. The remaining panels contrast the ground truth spatiotemporal solution with reconstructions from ProFlow and baseline methods (ECI, DiffusionPDE, D-Flow, PCFM). ProFlow accurately interpolates the shock propagation dynamics between observed timestamps, whereas baselines exhibit significant blurring or spurious oscillations in the unobserved intervals.}\label{fig:expr_sparse_obs}
\par\vspace{12pt}
\begingroup\captionsetup{type=table}

\renewcommand{\arraystretch}{.9988}\captionsetup{skip=3.5992pt}\caption{Burgers' trajectory reconstruction from sparse temporal observations. Methods must recover the full spatiotemporal solution $u_{0:T}$ given spatial measurements at only five random timestamps. Reported metrics are reconstruction error (RE), mean MSE (MMSE), standard deviation MSE (SMSE), and PDE error (PDE Err.). Lower values indicate better performance. The best value in each row is shown in boldface and the second best is underlined.}\label{tab:timestep}
\setlength{\tabcolsep}{3.04pt}
\begin{cicptabular}{l|l|c|c|c|c|c}\hline
Condition & Metric & ProFlow & ECI~\citep{cheng2025gradientfree} & DiffusionPDE~\citep{huang2024diffusionpde} & D-Flow~\citep{ben2024d} & PCFM~\citep{utkarsh2025physics}\\\hline
\multirow{4}{*}{Sparse Obs.} & RE ($\times10^{-3}$)&\textbf{1.33}&5.04&1.65&\underline{1.53}&30.5\\
 & MMSE ($\times10^{-4}$)&\textbf{1.73}&7.50&3.55&\underline{2.85}&232\\
 & SMSE ($\times10^{-4}$)&\textbf{3.80}&14.82&7.84&\underline{5.61}&28.8\\
 & PDE Err. ($\times10^{-4}$)&\textbf{2.77}&3.14&4.04&\underline{3.06}&5.38\\\hline
\end{cicptabular}
\endgroup
\end{figure}

\subsection{Recovering solutions through sparse temporal observation}
We adopt the experimental setup of DiffusionPDE~\citep{huang2024diffusionpde} applied to the $1$-dimensional dynamic Burgers' equation. In this setting, we assume access to spatial observations at only $5$ randomly sampled timestamps within the interval $[0,T]$ and aim to reconstruct the complete spatiotemporal solution $u_{0:T}$. \Cref{tab:timestep} demonstrates that ProFlow outperforms all baseline methods, achieving the lowest errors across every reported metric. Specifically, our method attains a reconstruction error of $1.33\times10^{-3}$ and a PDE residual of $2.77\times10^{-4}$, surpassing the strongest baseline, D-Flow, by approximately $10\%$. Notably, projection-based methods like ECI and PCFM struggle to bridge the large time gaps, leading to significant deviations in distributional statistics. As illustrated in \cref{fig:expr_sparse_obs}, this quantitative advantage translates to a sharp reconstruction of the shock propagation dynamics, where ProFlow avoids the blurring or spurious oscillations exhibited by baseline methods.

\begingroup\setlength{\intextsep}{12.256494pt}

\endgroup

\section{Conclusion and future work}\label{sec:conclusion}
In this work, we introduced ProFlow, a principled framework for zero-shot physics-consistent sampling using pretrained Functional Flow Matching priors. Motivated by the limitations of existing sampling-time methods which often struggle to balance strict constraint satisfaction with generative fidelity, ProFlow formulates the sampling process as a sequence of constrained posterior updates. By alternating between a terminal proximal optimization step that enforces physical and observational consistency, and an interpolation step that restores compatibility with the learned flow probability path, our method provides a rigorous mechanism to navigate the intersection of the solution manifold and the data manifold.

We provided a Bayesian interpretation of this procedure as a local MAP estimation scheme and validated its performance across a diverse suite of benchmarks, including forward, inverse, and joint reconstruction tasks for Poisson, Helmholtz, Darcy, and viscous Burgers' equations. Our numerical results demonstrate that ProFlow consistently outperforms state-of-the-art diffusion- and flow-based baselines, achieving superior reconstruction accuracy and lower PDE residuals while maintaining accurate distributional statistics.

Several directions for future work remain. First, while ProFlow avoids retraining, the computational cost of solving a proximal optimization problem at each sampling step is non-negligible. Future research could investigate accelerated splitting schemes to speed up the terminal refinement step. Second, extending this framework to turbulent or chaotic systems remains an important challenge. Scaling ProFlow to these regimes may require integrating latent-space flow matching models or multi-resolution strategies to handle the increased spatial complexity. Finally, while our method admits a local Bayesian interpretation, establishing global convergence guarantees for the coupled proximal-flow iteration represents a valuable theoretical open problem. Addressing these challenges will further advance the capability of generative models to serve as reliable, rigorous tools for scientific discovery.

\ProductionCompactSection*{Acknowledgments}
Wenyi Zhang acknowledges support from the National Key R\&D Program of China, grant No. \mbox{2025YFF0514401}, and the National Natural Science Foundation of China, grant No. \mbox{62231022}. Difan Zou acknowledges support from the Research Grants Council of the Hong Kong Special Administrative Region, China (General Research Fund), grant No. \mbox{17307425}, and the Natural Science Foundation of Guangdong Province, grant No. \mbox{2024A1515012444}. Weiguo Gao acknowledges support from the National Natural Science Foundation of China, grant No. \mbox{92670202}.
\appendix
\ProductionAppendixHead
\crefalias{section}{appendix}
\ProductionAppendixASection{Implementation details of pretrained FFM}\label{sec:pretrained model-details}
\paragraph{Network architecture.}
For all PDE benchmarks we employ a Fourier Neural Operator (FNO)~\citep{li2021fourier} as the velocity field estimator $\bm v_\theta(\bm u_t,t)$, following the architectural configuration used in ECI~\citep{cheng2025gradientfree}. Concretely, we use an FNO backbone with $L=4$ Fourier layers and a hidden width of $d_h=64$. The network first lifts the input channels to a higher-dimensional feature space of width $d_{\mathrm{lift}}=256$, applies a stack of spectral convolution blocks with a fixed number of retained Fourier modes in each spatial direction, and finally projects back to the physical channel dimension. This design gives a globally expressive yet translation-equivariant architecture that is well suited for learning solution operators of PDEs~\citep{li2021fourier}.

\paragraph{Input representation.}
The model takes as input the noisy state $\bm u_t$, concatenated with a time embedding and spatial positional encodings. In code, $\bm u_t\in\mathbb{R}^{B\times H\times W\times C}$ is first permuted to channel-first format and then augmented with additional channels. The time variable $t$ is represented either as a scalar or a batch-wise vector and is mapped to a time embedding of dimension $d_{\mathrm{emb}}=32$ using Gaussian Fourier features. This embedding is broadcast to the spatial grid so that each spatial location receives the same temporal feature. Spatial coordinates are normalized to the domain $[0,1]^d$ and discretized on the same grid as $\bm u_t$. We stack the coordinate grids along the channel dimension to form positional encodings. The final input to the FNO therefore has $C+d_{\mathrm{emb}}+d$ channels, corresponding to the physical state, the time embedding, and the spatial coordinates.

\paragraph{Training configuration.}
We train the FNO velocity field using an optimal transport conditional flow matching objective with a linear noise schedule, consistent with recent work on flow matching models~\citep{lipman2023flow,liu2023flow}. Each mini-batch is drawn from a PDE-specific training set that provides clean solution fields $\bm u_1$ on a fixed spatial grid. For every clean sample we independently draw an initial noise realization $\bm u_0\sim\mathcal{N}(\bm0,\bm I)$ that matches the spatial resolution and channel count of $\bm u_1$. We then sample a time $t\sim\mathcal{U}(0,1)$ and construct the linear interpolant
\begingroup\setlength{\abovedisplayskip}{6pt}\setlength{\belowdisplayskip}{6pt}\setlength{\belowdisplayshortskip}{6pt}
\[
\bm u_t=(1-t)\bm u_0+t\bm u_1,
\]
which follows the straight path between noise and data. This corresponds to the rectified flow setting where the target probability path is the linear interpolation between the prior and the data distribution.

Given $\bm u_t$ and $t$, we evaluate the FNO $\bm v_\theta(\bm u_t,t)$ and regress it toward the analytical conditional velocity of the straight path. For a linear interpolation the target vector field is the time-independent difference $\bm v_t^\star(\bm u_t)=\bm u_1-\bm u_0$~\citep{lipman2023flow,liu2023flow}. The flow matching loss for a batch is therefore
\[
\mathcal{L}_{\mathrm{FM}}=\mathbb{E}_{\bm u_0,\bm u_1,t}\bigl\|\bm v_\theta(\bm u_t,t)-(\bm u_1-\bm u_0)\bigr\|_2^2,
\]
which we implement as a mean squared error between the predicted and target velocities over all spatial locations and channels. This simulation-free regression objective avoids backpropagation through an ODE solver and reduces to standard supervised learning over randomly sampled triples $(\bm u_0,\bm u_1,t)$.

\endgroup
Our default configuration uses the Adam optimizer with a learning rate of $3\times10^{-4}$ and a batch size of $128$. We train for $10{,}000$ iterations on each PDE benchmark, which corresponds to $10{,}000$ gradient updates over mini-batches sampled from the training set. The main training loop repeatedly draws a batch from the data loader, samples noise and times, constructs the interpolated states, computes the flow matching loss, and updates the network parameters.

\ProductionCompactSection{Implementation details of baseline methods}\label{sec:details_baseline}
We provide detailed configurations for each baseline considered in this paper. When necessary, we adapt each method within our ProFlow framework, incorporating their respective physical-constraint mechanisms to ensure faithful implementation and comparable performance. Unless otherwise specified, we employ an explicit Euler integrator with $100$ steps for all experiments to ensure consistent evaluation.

\paragraph{ECI sampling.}
We follow the ECI sampling~\citep{cheng2025gradientfree} procedure, which performs iterative extrapolation--correction--interpolation updates throughout the sampling trajectory. At each step we first compute the one-step terminal prediction (analogous to the Tweedie estimate in diffusion) as $\hat{\bm u}_1=\bm u_t+(1-t)\bm v_\theta(\bm u_t,t)$. We then impose the observation mask by projecting this prediction back onto the constraint manifold using a correction operator $\mathcal{P}_{\mathcal{C}}(\cdot)$. The resulting update maps the corrected estimate back to the intermediate time $t+\Delta t$ using
\begingroup\setlength{\abovedisplayskip}{6pt}\setlength{\belowdisplayskip}{6pt}\setlength{\belowdisplayshortskip}{6pt}
\begin{align}
\bm u_{t+\Delta t}=(t+\Delta t)\mathcal{P}_{\mathcal{C}}(\hat{\bm u}_1)+(1-t-\Delta t)\bigl(\bm u_t+\bm v_\theta(\bm u_t,t)\Delta t\bigr).\label{eq:eci-update}
\end{align}\endgroup
\begingroup\tolerance=1000\looseness=1
To ensure a consistent configuration across all PDEs considered, we apply $n_{\mathrm{mix}}$ ECI \emph{mixing} iterations at each Euler step. A single mixing iteration reuses the same time index and repeatedly applies the extrapolation--correction--interpolation update in \eqr{eq:eci-update}, feeding the corrected field back into the flow ODE. This recursive application promotes information exchange between constrained and unconstrained regions and empirically reduces artifacts in the generated trajectories~\citep{cheng2025gradientfree}. The hyperparameter $n_{\mathrm{mix}}$ therefore controls how aggressively the constraint is propagated at each time step, and in all our experiments we set $n_{\mathrm{mix}}=5$. We further control the stochasticity of ECI sampling by resampling the Gaussian prior noise used in the interpolation step, following the re-sampling strategy of Cheng et al.~\citep{cheng2025gradientfree}. For more implementation details, please refer to \url{https://github.com/amazon-science/ECI-sampling/tree/main}.\par\endgroup

\paragraph{DiffusionPDE.}
DiffusionPDE~\citep{huang2024diffusionpde} augments each sampling step with soft guidance derived from both observational consistency and PDE residual minimization. Let $\bm{u}_t$ denote the evolving PDE state and $\bm{v}_\theta(\bm{u}_t,t)$ be the pretrained flow velocity field. At each integration step we update the state according to
\begin{align}
\bm{u}_{t+\Delta t}=\bm{u}_t+\bm{v}_\theta(\bm{u}_t,t)\Delta t-\alpha_t\nabla_{\bm{u}_t}L_{\mathrm{obs}}(\bm{u}_t)-\beta_t\nabla_{\bm{u}_t}L_{\mathrm{PDE}}(\bm u_t),\label{eq:diffusionpde_update}
\end{align}
where the first term $\bm{u}_t+\bm{v}_\theta(\bm{u}_t,t)\Delta t$ corresponds to an explicit Euler step and the remaining two terms subtract gradients of task-specific loss functionals. To encode observational consistency we define the observation loss
\[
L_{\mathrm{obs}}(\bm{u}_t)=\bigl\|\mathcal{H}(\bm{u}_t)-\bm{y}\bigr\|^2,
\]
where $\mathcal{H}$ maps the full state $\bm{u}_t$ to the measurement space, and $\bm{y}$ denotes the measured data. The gradient $\nabla_{\bm{u}_t}L_{\mathrm{obs}}(\bm{u}_t)$ therefore implements a standard data-fidelity step that pulls the current sample toward agreement with the observations. For the physics term we write each PDE in residual form $\mathcal{R}(\bm{u})(\bm{x})=0$, where $\mathcal{R}$ collects the differential operator, coefficients, and forcing terms evaluated at each spatial and possibly temporal location. Following DiffusionPDE~\citep{huang2024diffusionpde} we do not apply the PDE loss directly to the noisy state $\bm{u}_t$ but instead to the posterior mean predictor $\hat{\bm{u}}_1(\bm{u}_t,t)$, which plays the role of a denoised estimate at the terminal time. This yields the PDE loss
\[
L_{\mathrm{PDE}}(\bm{u}_t)=\bigl\|\mathcal{R}\bigl(\hat{\bm{u}}_1(\bm{u}_t,t)\bigr)\bigr\|^2,
\]
whose gradient $\nabla_{\bm{u}_t}L_{\mathrm{PDE}}(\bm{u}_t)$ is obtained by backpropagating through both the PDE discretization used to evaluate $\mathcal{R}$ and the denoising network that produces $\hat{\bm{u}}_1$. The scalar schedules $\alpha_t$ and $\beta_t$ control the relative strengths of the observation and physics guidance throughout the trajectory. For each PDE family and each experimental task we tune these schedules via grid search and report the best achieved performance in the main results. Apart from this choice of guidance weights we follow the implementation details of DiffusionPDE, including the finite-difference and finite-element stencils used to instantiate $\mathcal{R}$ for each equation, and refer the reader to the official code release for the exact discretizations and hyperparameters: \url{https://github.com/jhhuangchloe/DiffusionPDE}.\looseness=1

\paragraph{D-Flow.}
D-Flow~\citep{ben2024d} casts constraint enforcement as an optimization problem posed directly on the initial noise $\bm u_0$. Instead of modifying the sampling dynamics at intermediate times, the method searches for an initial condition whose deterministic flow trajectory simultaneously matches the observations and satisfies the governing PDE. Let $\Phi_\theta\colon\bm u_0\mapsto\bm u_1$ denote the deterministic flow map obtained by integrating the learned velocity field $\bm v_\theta$ from $t=0$ to $t=1$. D-Flow determines the optimal initialization $\bm u_0^\star$ by minimizing a joint objective defined on the terminal state:
\[
\bm u_0^\star=\operatorname*{argmin}_{\bm u_0}\Bigl\{\|\mathcal{H}[\Phi_\theta(\bm u_0)]-\bm y\|_2^2+\gamma\|\mathcal{R}(\Phi_\theta(\bm u_0))\|_2^2\Bigr\},
\]
where the first term enforces consistency with the observations $\bm y$ under the observation operator $\mathcal{H}$, and the second term penalizes violations of the PDE residual $\mathcal{R}$ evaluated on the terminal field. The scalar hyperparameter $\gamma$ balances the relative importance of data fidelity and physical consistency. Following the official implementation and the experimental protocol adopted in ECI~\citep{cheng2025gradientfree}, we approximate the continuous-time flow map $\Phi_\theta$ by discretizing the ODE into $S=100$ explicit Euler steps for all tasks unless otherwise specified. Concretely, given a candidate initialization $\bm u_0$, we construct the terminal state $\Phi_\theta(\bm u_0)$ by repeatedly applying
\[
\bm u_{s+1}=\bm u_s+\bm v_\theta(\bm u_s,t_s)\Delta t,\quad s=0,\ldots,S-1,
\]
with a uniform time grid $0=t_0<t_1<\cdots<t_S=1$ and $\Delta t=1/S$. The loss in the optimization is then computed on $\bm u_S=\Phi_\theta(\bm u_0)$ by evaluating both the observation discrepancy $\|\mathcal{H}(\bm u_S)-\bm y\|_2^2$ and the PDE residual norm $\|\mathcal{R}(\bm u_S)\|_2^2$. We solve the resulting finite-dimensional optimization problem over $\bm u_0$ with L-BFGS~\citep{liu1989limited}, using $K=20$ iterations and a learning rate of $0.1$. Each L-BFGS iteration requires both a forward integration of the flow ODE to evaluate the objective and a corresponding backward pass to compute $\nabla_{\bm u_0}$ of the loss. In practice, we obtain these gradients by differentiating through the full $S$-step trajectory using the adjoint-based ODE differentiation interface from \texttt{torchdiffeq}~\citep{chen2018neural}, which reconstructs the necessary intermediate states on the fly during the backward pass. This choice avoids storing all intermediate $\bm u_s$ in memory, but it further increases the computational burden because each gradient evaluation entails reintegrating the ODE in reverse time. The overall cost of D-Flow scales linearly in both the number of Euler steps $S$ used to discretize $\Phi_\theta$ and the number of optimization iterations $K$. Each optimization step invokes $S$ model evaluations in the forward pass and an additional $S$ evaluations in the adjoint-based backward pass, resulting in approximately $2SK$ evaluations of $\bm v_\theta$ per sample. For simplicity we report this as $\mathcal{O}(SK)$ model evaluations and, following the common convention in diffusion and flow-matching literature, we count them as $S\times K$ function evaluations (NFE). For instance, setting $S=100$ and $K=20$ as in our Burgers' experiments leads to roughly $2{,}000$ NFE per constrained sample, which is substantially more expensive than a single ProFlow trajectory under the same backbone model. For more implementation details, including the exact parameterization of $\gamma$ and the stopping criteria, please refer to the public D-Flow notebook at
\url{https://colab.research.google.com/drive/1R6X3xSlb-mtV7QTaDEj6VtOf2GZyTF8i?usp=sharing#scrollTo=un_zmDDx3Ktu}.

\paragraph{PCFM.}
In our implementation, we closely follow the official PCFM~\citep{utkarsh2025physics} codebase, which realizes Physics-Constrained Flow Matching as an explicit Euler integrator augmented with constraint corrections at every simulation step. Let $\bm u_t$ denote the current flow-matching state and let $\mathcal{C}(\bm u)$ be the constraint residual, for example the deviation from a prescribed integral conservation law or boundary condition. The corresponding constraint manifold is
\begingroup\setlength{\abovedisplayskip}{6pt}\setlength{\belowdisplayskip}{5pt}\setlength{\belowdisplayshortskip}{5pt}
\[
\mathcal{M}=\{\bm u\mid\mathcal{C}(\bm u)=\bm0\}.
\]
At each time step, we first take a standard flow-matching Euler update driven by the pretrained velocity field $\bm v_\theta$,
\predisplaypenalty=0
\[
\bm u_{t+\Delta t}^{\mathrm{euler}}=\bm u_t+\bm v_\theta(\bm u_t,t)\Delta t,
\]
which advances the sample along the generative trajectory without enforcing any constraints. We then apply a single Gauss--Newton correction step to project this intermediate state back toward the constraint manifold. Concretely, we linearize the constraint around $\bm u_{t+\Delta t}^{\mathrm{euler}}$ using the Jacobian $J_{\mathcal{C}}$ and solve a small linear system for a correction direction $\delta\bm u$ that reduces $\|\mathcal{C}(\bm u)\|_2$, then update
\[
\bm u_{t+\Delta t}\approx\bm u_{t+\Delta t}^{\mathrm{euler}}+\delta\bm u.
\]
\endgroup
Following the PCFM sampling setup, we also enable the guided interpolation refinement used in their experiments. PCFM represents backward transport between a corrected terminal state and an earlier state using an optimal-transport displacement interpolant, and then refines points along this interpolant with a relaxed constraint correction step~\citep{utkarsh2025physics}. In this refinement, a penalty parameter $\lambda$ balances staying close to the OT interpolant against reducing the constraint residual. We adopt the same configuration as the official implementation and set $\lambda=1.0$, use a gradient-based update with step size $0.01$, and perform $20$ refinement steps for each interpolated state. Intuitively, $\lambda$ controls how strongly the refinement step pulls samples back toward the constraint manifold. The step size $0.01$ determines how aggressively each refinement iteration moves in that direction. The $20$ refinement iterations allow the interpolant to converge to a state that both lies close to the flow-matching path and satisfies the imposed constraints to high accuracy. Finally, after completing all flow-matching steps, we follow PCFM and apply a last projection on the terminal state using the same Newton-style solver. This final correction enforces the constraints up to numerical precision while preserving the distribution learned by the pretrained flow model. For more implementation details, including the exact constraint definitions and solver configurations, please refer to the official PCFM repository at
\url{https://github.com/cpfpengfei/PCFM/tree/main}.\looseness=1

\ProductionCompactSection{Implementation details of ProFlow}\label{sec:proflow-details}
ProFlow performs standard Euler discretization of the learned flow ODE combined with constrained proximal optimization at every timestep. We start each trajectory from Gaussian noise $\bm u_0\sim\mathcal{N}(\bm0,\bm I)$ whose spatial resolution matches the target PDE field. At step $i$ with current state $\bm u_t$ at time $t=i/N$, we first query the pretrained flow model to obtain a one-step terminal prediction
\[
\hat{\bm u}_1=\bm u_t+(1-t)\bm v_\theta(\bm u_t,t),
\]
which we treat as a \emph{proximal anchor}. This prediction corresponds to the model's estimate of the clean field at $t=1$ given the current noisy state and plays the same role as the posterior mean predictor in diffusion-based samplers.

Given $\hat{\bm u}_1$, ProFlow solves a small proximal subproblem at each step. We introduce an observation mask $\bm m$ and observations $\bm c$ extracted from the training set, and define a joint objective on the terminal field
\begin{align*}
L(\bm u_1)=\|\bm u_1-\hat{\bm u}_1\|_2^2+\lambda_{\mathrm{obs}}\bigl\|(\bm u_1-\bm c)\odot\bm m\bigr\|_2^2+\lambda_{\mathrm{pde}}\bigl\|\mathcal{R}(\bm u_1)\bigr\|_2^2,
\end{align*}
where $\odot$ denotes the entrywise multiplication. The first term keeps the refined field close to the flow prediction, which preserves the learned prior. The second term enforces data fidelity only on the observed region through the mask $\bm m$. The third term penalizes the PDE residual $\mathcal{R}(\bm u_1)$ evaluated using the same discretization as in the baseline methods. In our default configuration we use $\lambda_{\mathrm{obs}}=80$ and $\lambda_{\mathrm{pde}}=10^{-3}$ for the Darcy benchmark and reuse the same structure with appropriately tuned weights for other PDEs.

We minimize $L(\bm u_1)$ using a customized first-order iterative solver implemented directly on the terminal field. Starting from $\bm u_1^{(0)}=\hat{\bm u}_1$, we perform $K=3$ gradient descent updates
\[
\bm u_1^{(k+1)}=\bm u_1^{(k)}-\eta_t\nabla_{\bm u_1}L\bigl(\bm u_1^{(k)}\bigr),\quad k=0,\ldots,K-1,
\]
where $\eta_t=\eta_0s(t)$ is a time-dependent step size given by a simple learning rate schedule that decreases as $t$ approaches $1$. In practice this schedule scales roughly with $1-t$, which allocates larger steps in the early noisy regime and smaller, more conservative updates near the clean terminal state. All gradients are computed with automatic differentiation, and we detach the refined field inside the inner loop to keep the computational graph local to each proximal step. We emphasize that this is an \emph{inexact} proximal solver, and we can make precise the sense in which $K=3$ steps suffice. In the forward and inverse elliptic tasks the discretized residual is affine in the unknown field, since the diffusion term $-\nabla\cdot(\bm a\nabla\bm u)$ is linear in $\bm u$ for fixed $\bm a$ and linear in $\bm a$ for fixed $\bm u$, and the mask is linear. There $L$ is a convex quadratic with Hessian
\[
\bm H=2\bigl(\bm I+\lambda_{\mathrm{obs}}\operatorname{diag}(\bm m)+\lambda_{\mathrm{pde}}\bm A^\top\bm A\bigr),
\]
where $\bm A$ is the residual Jacobian. The anchor term contributes the $2\bm I$, so $L$ is $\mu$-strongly convex with $\mu\geq2$ and its gradient is $\ell$-Lipschitz with $\ell\leq2\bigl(1+\lambda_{\mathrm{obs}}+\lambda_{\mathrm{pde}}\|\bm A^\top\bm A\|_2\bigr)$. Gradient descent with $0<\eta_t<2/\ell$ then contracts linearly,
\[
\|\bm u_1^{(k)}-\bm u_1^\star\|_2\leq\rho^k\|\hat{\bm u}_1-\bm u_1^\star\|_2
\]
with $\rho=\max(|1-\eta_t\mu|,|1-\eta_t\ell|)<1$, so $K$ steps leave an inner error of at most $\rho^K\|\hat{\bm u}_1-\bm u_1^\star\|_2$. This error is small in absolute terms because the anchor term keeps the minimizer near the warm start, with $\|\hat{\bm u}_1-\bm u_1^\star\|_2$ itself bounded by the weighted data and physics misfit of the flow prediction, so a fixed contraction already yields a small residual. The correction is moreover \emph{amortized} over the $N$ outer steps, so ProFlow performs many mild, well-conditioned corrections across the schedule, and any error left by an early step is further reduced as $t\to1$ and the anchor sharpens. We therefore claim not per-step optimality but a controlled inexact solve with contraction $\rho^K$ per outer step, trading exactness for a large saving in cost, namely three gradient evaluations per step and no Jacobian assembly. In the joint elliptic setting, where the diffusion term is bilinear in the unknown pair $(\bm a,\bm u)$, and for the nonlinear viscous Burgers' residual, $L$ is smooth but generally nonconvex and the linear rate holds only locally. There the same $K=3$ updates act as a local descent from the warm start, and we rely on the amortization above together with the empirical residual reductions in \Cref{sec:experiments}. Increasing $K$ or $\lambda_{\mathrm{pde}}$ tightens the per-step solve at proportionally higher cost and recovers the exact proximal step as $K\to\infty$.

After $K$ proximal iterations we obtain a refined terminal field $\tilde{\bm u}_1$. Then we define the next state $\bm u_{t'}$ (with $t'=(i+1)/N$) by mixing the refined data with fresh noise $\bm\varepsilon\sim\mathcal{N}(\bm0,\bm I)$. We apply the linear interpolation formula
\[
\bm u_{t'}=(1-t')\bm\varepsilon+t'\bm u_1,
\]
which projects the iterate back onto the probability path used during training. This step prevents manifold divergence, ensuring the velocity network $\bm v_\theta$ operates on valid samples during the subsequent integration step.

This procedure gives ProFlow a simple and efficient implementation. Each sampling step consists of one evaluation of the flow model to obtain $\hat{\bm u}_1$ and the velocity field, a small number $K=3$ of proximal gradient updates on the terminal field, and one interpolation call that maps the refined terminal state back to the intermediate time. For detailed experimental settings and reproduction scripts, including the exact implementations of the observation mask, PDE residual operator, and learning rate schedule, the code is publicly available at \url{https://github.com/yzc-666/ProFlow}.

\ProductionCompactSection{Additional experimental results}\label{sec:additional-experimental-results}
Beyond the accuracy and distributional comparisons in \Cref{sec:experiments}, we further evaluate ProFlow along two complementary axes: (i) the computational cost of sampling relative to all baselines, and (ii) the sensitivity of ProFlow to its principal sampling hyperparameters. All experiments use the same pretrained FFM priors, datasets, and observation settings as in \Cref{sec:experiments}.

\ProductionModerateSubsection{Computational cost}\label{subsec:computational cost}
We measure the end-to-end wall-clock time to generate one conditional sample (batch size 1, $N=100$ Euler steps) for each method on four representative tasks. All measurements are performed on a single NVIDIA A800-80GB GPU. For every (task, method) pair we perform 10 warmup runs followed by 50 timed runs on independent test instances, with GPU synchronization immediately before and after each timed call. The timing includes all neural network evaluations, PDE-residual evaluations, gradient computations, constraint corrections, and inner optimization steps, while model and data loading are excluded. Peak GPU memory is recorded via the CUDA allocator statistics. Owing to its very high per-sample cost, PCFM is averaged over 5 runs. Results are summarized in Table~\ref{tab:cost}.
ProFlow is the fastest method on every task ($0.56$--$0.86$\,s per sample), despite solving a proximal subproblem at each step: the inner solver requires only $K=3$ lightweight gradient updates on the terminal field and no Jacobian assembly. DiffusionPDE is comparable in cost ($0.88$--$0.93$\,s), while ECI is roughly $2$--$3\times$ slower due to its $n_{\mathrm{mix}}=5$ mixing iterations. D-Flow is $7$--$11\times$ slower than ProFlow because each sample requires an L-BFGS optimization that differentiates through the full ODE trajectory. PCFM is two to three orders of magnitude more expensive ($73$--$560$\,s), as its Gauss--Newton correction assembles PDE-residual Jacobians at every step. This is also reflected in its peak memory footprint (up to $17.4$\,GB versus $\leq0.31$\,GB for ProFlow). Overall, ProFlow attains the accuracy benefits reported in \Cref{sec:experiments} at the lowest sampling cost among all constrained baselines.

\begin{table}[!t]\centering\renewcommand{\arraystretch}{1.225}
\captionsetup{skip=5.3542pt}\caption{Computational cost per generated sample (batch size 1, $N=100$ Euler steps). Timing uses 10 warmup runs followed by 50 measured runs with GPU synchronization before and after each call, and it includes all network evaluations, PDE-residual evaluations, gradients, and constraint corrections. PCFM is averaged over 5 runs owing to its very high per-sample cost. Runtime is reported as mean $\pm$ std.}\label{tab:cost}
\begin{cicptabular}{l|c|c|c|c}\hline
Method & Poisson Fwd. & Helmholtz Fwd. & Darcy Fwd. & Burgers' Sparse Obs.\\\hline
\multicolumn{5}{||c||}{\emph{Runtime per sample (s)}}\\\hline
ECI&$1.66\pm0.01$&$1.66\pm0.01$&$1.70\pm0.10$&$1.72\pm0.01$\\
DiffusionPDE&$0.92\pm0.02$&$0.93\pm0.02$&$0.90\pm0.02$&$0.88\pm0.02$\\
D-Flow&$6.31\pm0.08$&$6.17\pm0.05$&$6.18\pm0.04$&$6.34\pm0.06$\\
PCFM&$560.07\pm0.15$&$560.01\pm0.11$&$560.45\pm0.13$&$73.55\pm0.06$\\
ProFlow (Ours)&$0.86\pm0.03$&$0.85\pm0.02$&$0.77\pm0.02$&$0.56\pm0.03$\\\hline
\multicolumn{5}{||c||}{\emph{Peak GPU memory (GB)}}\\\hline
ECI&0.14&0.15&0.15&0.14\\
DiffusionPDE&0.25&0.24&0.25&0.24\\
D-Flow&0.15&0.15&0.15&0.14\\
PCFM&17.38&17.38&17.38&4.49\\
ProFlow (Ours)&0.31&0.31&0.31&0.30\\\hline
\end{cicptabular}
\end{table}

\begin{table}[!tbp]\centering\renewcommand{\arraystretch}{1.0}
\captionsetup{skip=5.3542pt}\caption{Sensitivity of ProFlow to its principal sampling hyperparameters ($N$, $K$, $\lambda_{\mathrm{obs}}$, $\lambda_{\mathrm{pde}}$) on Darcy joint sparse reconstruction and Burgers' trajectory reconstruction from five temporal snapshots. Each block varies one parameter while holding the others at their defaults ($N=100$, $K=3$, and the PDE-specific $\lambda_{\mathrm{obs}}^0$, $\lambda_{\mathrm{pde}}^0$). Values are averaged over the same 50 test instances. Lower is better.}\label{tab:sensitivity}
\setlength{\tabcolsep}{5.45pt}
\begin{cicptabular}{l|c|c|c|c|c|c|c}\hline
\multicolumn{2}{||c|}{} & \multicolumn{3}{c|}{Darcy} & \multicolumn{3}{c||}{Burgers}\\\cline{3-8}
Parameter & Value & RE & PDE Err. & Time (s) & RE & PDE Err. & Time (s)\\\hline
\multirow{4}{*}{$N$}&25&$5.42\times10^{-2}$&$2.27\times10^{-4}$&0.20&$1.52\times10^{-3}$&$3.18\times10^{-4}$&0.12\\
&50&$5.39\times10^{-2}$&$6.23\times10^{-5}$&0.39&$1.44\times10^{-3}$&$3.10\times10^{-4}$&0.25\\
&100&$5.34\times10^{-2}$&$2.74\times10^{-5}$&0.79&$1.33\times10^{-3}$&$2.77\times10^{-4}$&0.59\\
&200&$5.29\times10^{-2}$&$3.07\times10^{-5}$&1.58&$1.33\times10^{-3}$&$2.37\times10^{-4}$&1.08\\\hline
\multirow{4}{*}{$K$}&1&$5.34\times10^{-2}$&$2.74\times10^{-5}$&0.51&$1.33\times10^{-3}$&$2.77\times10^{-4}$&0.52\\
&3&$5.34\times10^{-2}$&$2.74\times10^{-5}$&0.79&$1.33\times10^{-3}$&$2.77\times10^{-4}$&0.59\\
&5&$5.34\times10^{-2}$&$2.74\times10^{-5}$&1.06&$1.33\times10^{-3}$&$1.55\times10^{-4}$&0.97\\
&10&$6.57\times10^{-2}$&$5.45\times10^{-5}$&1.75&$4.20\times10^{-3}$&$1.59\times10^{-5}$&1.50\\\hline
\multirow{5}{*}{$\lambda_{\mathrm{obs}}/\lambda_{\mathrm{obs}}^0$}&0.25&$5.54\times10^{-2}$&$2.66\times10^{-5}$&0.86&$7.42\times10^{-4}$&$6.23\times10^{-5}$&0.51\\
&0.5&$5.40\times10^{-2}$&$2.55\times10^{-5}$&0.80&$7.48\times10^{-4}$&$1.10\times10^{-4}$&0.59\\
&1&$5.34\times10^{-2}$&$2.74\times10^{-5}$&0.79&$1.33\times10^{-3}$&$2.77\times10^{-4}$&0.59\\
&2&$5.40\times10^{-2}$&$3.34\times10^{-5}$&0.79&$2.55\times10^{-3}$&$6.85\times10^{-4}$&0.51\\
&4&$5.47\times10^{-2}$&$3.54\times10^{-5}$&0.80&$4.95\times10^{-3}$&$1.72\times10^{-3}$&0.51\\\hline
\multirow{5}{*}{$\lambda_{\mathrm{pde}}/\lambda_{\mathrm{pde}}^0$}&0.25&$5.34\times10^{-2}$&$2.74\times10^{-5}$&0.81&$1.33\times10^{-3}$&$3.35\times10^{-4}$&0.51\\
&0.5&$5.34\times10^{-2}$&$2.74\times10^{-5}$&0.80&$1.34\times10^{-3}$&$3.15\times10^{-4}$&0.51\\
&1&$5.34\times10^{-2}$&$2.74\times10^{-5}$&0.79&$1.33\times10^{-3}$&$2.77\times10^{-4}$&0.59\\
&2&$5.34\times10^{-2}$&$2.74\times10^{-5}$&0.81&$1.33\times10^{-3}$&$2.26\times10^{-4}$&0.54\\
&4&$5.34\times10^{-2}$&$2.74\times10^{-5}$&0.80&$1.32\times10^{-3}$&$2.04\times10^{-4}$&0.53\\\hline
\end{cicptabular}
\end{table}
\ProductionCompactSubsection{Hyperparameter sensitivity}\label{subsec:hyperparameter sensitivity}
\begingroup\tolerance=800\looseness=-1
We next study the robustness of ProFlow with respect to its principal sampling hyperparameters through a one-at-a-time sensitivity analysis of the number of interpolation steps $N$ (equivalently the step size $\Delta t=1/N$), the number of inner terminal updates $K$, the observation weight $\lambda_{\mathrm{obs}}$, and the PDE-residual weight $\lambda_{\mathrm{pde}}$. The analysis is conducted on two representative tasks: Darcy joint sparse reconstruction with $50\%$ of both fields observed (representing the elliptic benchmarks), and Burgers' trajectory reconstruction from five randomly selected temporal snapshots (representing a nonlinear time-dependent problem). For each configuration we use the same 50 test instances, observation masks, and random seeds, and report reconstruction error (RE), PDE error, and sampling time per sample. Unless the parameter under study is varied, all remaining parameters are held at their defaults $N=100$, $K=3$, and the PDE-specific weights $\lambda_{\mathrm{obs}}^0$ and $\lambda_{\mathrm{pde}}^0$ (for Darcy, $\lambda_{\mathrm{obs}}^0=80$ and $\lambda_{\mathrm{pde}}^0=10^{-3}$). To make the weight study comparable across PDE families, each weight is varied by a multiplicative factor relative to its default. Results are reported in Table~\ref{tab:sensitivity}.\par\endgroup

ProFlow is stable over a wide range of settings. Increasing $N$ improves accuracy only marginally beyond $N=100$ while the cost grows linearly, supporting $N=100$ as a favorable accuracy--cost trade-off. Even $N=25$ degrades RE by less than $15\%$ on both tasks. The method is likewise insensitive to the number of inner updates for $K\in\{1,3,5\}$. Only the aggressive setting $K=10$ over-corrects the terminal state and noticeably degrades reconstruction, which is consistent with our design choice of a small number of mild proximal corrections amortized over the trajectory (\Cref{sec:proflow-details}). Varying $\lambda_{\mathrm{pde}}$ over a $16\times$ range leaves RE essentially unchanged while larger values further reduce the PDE residual, indicating that the physics penalty can be strengthened without sacrificing observational fidelity. The observation weight $\lambda_{\mathrm{obs}}$ is the most sensitive parameter on the nonlinear Burgers' task, where overly large values ($2$--$4\times$ the default) degrade RE. The elliptic Darcy task remains stable within $\pm4\%$ across the entire range. In summary, no parameter requires fine-tuning: performance varies smoothly and the defaults lie in a broad, flat region of the hyperparameter landscape.

\end{document}